\documentclass{article}

\usepackage[preprint]{neurips_2026}

\usepackage[utf8]{inputenc}
\usepackage[T1]{fontenc}
\usepackage{hyperref}
\usepackage{url}
\usepackage{booktabs}
\usepackage{amsfonts}
\usepackage{amsmath}
\usepackage{amssymb}
\usepackage{amsthm}
\usepackage{subcaption}
\usepackage{bm}
\usepackage{nicefrac}
\usepackage{microtype}
\usepackage{xcolor}
\usepackage{graphicx}
\usepackage{enumitem}
\usepackage{mathtools}
\usepackage{xcolor}

\newtheorem{proposition}{Proposition}

\title{Take It or Leave It: Intent-Controlled Partial Optimal Transport}

\author{%
  Salil Parth Tripathi\\
  OceanDataLab\\
  \texttt{parth.tripathi@oceandatalab.com}
  \And
  Bertrand Chapron\\
  Ifremer\\
  \texttt{bertrand.chapron@ifremer.fr}
  \And
  Fabrice Collard\\
  OceanDataLab\\
  \texttt{dr.fab@oceandatalab.com}
  \And
  Nicolas Courty\\
  Universit\'e Bretagne Sud\\
  \texttt{nicolas.courty@univ-ubs.fr}
  \And
  Ronan Fablet\\
  IMT Atlantique\\
  \texttt{ronan.fablet@imt-atlantique.fr}
}

\begin{document}

\maketitle

\begin{abstract}
While optimal transport (OT) enforces a rigid constraint by requiring two measures to be matched exactly, partial optimal transport relaxes this requirement by allowing mass to remain unmatched through a global budget, scalar rebate, or uniform rejection rule. However, many applications call for more structured, pointwise rejection mechanisms, where the decision to leave mass unmatched depends on side-specific reliability, support geometry, or external information about which components should participate in the comparison. We introduce \emph{intent-controlled partial optimal transport} (IC-POT), a targeted generalization of partial transport that replaces the global rejection paradigm with pointwise rejection costs over both measures. We show that the resulting optimization problem admits a dual interpretation in terms of local acceptance thresholds and can be solved by recasting it as a balanced Kantorovich OT problem on an augmented support. Beyond theoretical analysis, we demonstrate the practical relevance of IC-POT in settings where rejection is driven by side information. In positive-unlabeled learning and open-partial domain adaptation, incorporating pointwise rejection rules that encode statistical structure improves fixed baseline pipelines. Finally, we motivate the use of IC-POT with a geophysical practical case: multi-modal satellite ocean measurements, for which physical and sensors priors naturally inform the rejection mechanism and define the retrieved comparable signal information.

\end{abstract}

\section{Introduction}
Optimal transport (OT) provides a principled way to compare and align distributions under a transport cost, and has become a standard tool in machine learning, computer vision, and scientific computing \citep{peyre2019computational}. In its classical balanced form, all source mass must be transported and all target mass must be explained. This full-participation assumption is often too rigid and may result in artificial correspondences, negative transfer, or spurious matches between structures that should play different roles in the comparison.

Partial and unbalanced OT introduce non-participation by allowing some mass to remain unmatched \citep{caffarelli2010free,chapel2020partial,liero2018optimal,chizat2018scaling}. In most formulations, however, this choice is governed by a global transported-mass budget, a scalar rebate, or a marginal penalty. These mechanisms only control how much mass is rejected, not which points should be protected, exposed, or made inexpensive to discard. Many applications require a formulation in which leaving mass unmatched can be priced locally, according to the information available on each support. In positive-unlabeled learning, rejection is tied to the sampling mechanism \citep{chapel2020partial}. In open-partial domain adaptation, it is tied to confidence and target geometry \citep{courty2017optimal,chang2022uniot,luo2023mot}. In geophysical comparison, the question is often not whether two observations differ, but which parts of them can be meaningfully compared given their measurement geometry and sensor-specific uncertainties \citep{hauser2021swim,hasselmann1985sar,stopa2015azimuthcutoff}. 

To address these challenges, we introduce \emph{intent-controlled partial optimal transport} (IC-POT), a targeted extension of partial OT formulated with explicit source-side and target-side slack variables. Pointwise costs attached to these slacks price the local alternative of leaving mass unmatched, while the pairwise transport cost continues to describe the geometry of matching. The modeling shift is simple: the unmatched policy becomes representable directly on the original supports. We show that this formulation 
%The formulation also has a direct mathematical position relative to classical partial transport. Eliminating the slack variables 
yields an equivalent Lagrangian sub-coupling problem in the sense of Caffarelli--McCann \citep{caffarelli2010free}, with the scalar transport rebate replaced by a separable local rebate, one term on each marginal. This reduced form preserves the side-specific interpretation of the unmatched costs and leads to dual acceptance caps, an exact admissibility rule for active transport edges, and a strict extension of the constant-cost partial OT model.
Experiments on positive-unlabeled learning, open-partial domain adaptation, and multimodal satellite-derived ocean-wave spectra show how IC-POT leverages side information to specify the unmatched policy. PU learning and OPDA test statistical settings where sampling bias, confidence, and latent geometry determine which samples should remain protected from rejection. The ocean-wave experiment compares spectra retrieved from SWIM, the wave scatterometer onboard CFOSAT \citep{hauser2021swim}, and Synthetic Aperture Radar (SAR) \citep{hasselmann1985sar}, and illustrates a distinct role of the formulation: physical priors naturally inform the rejection mechanism and define the comparable signal itself. Across these experiments, the transport geometry and baseline pipeline are kept fixed; what changes is the way unmatched mass is priced.

Overall, the contribution of the paper is threefold. First, it makes unmatched behavior representable as an explicit object in the formulation through a slack formulation on the original support. Second, it characterizes the resulting problem through an equivalent separable pointwise-rebate generalization of the classical Lagrangian partial transport, with Kantorovich equivalence, dual caps, sparse admissibility, and strict separation from constant-cost partial OT. Third, it shows empirically that this added modeling capacity matters in reject-structured tasks compared with constant-cost partial OT. %including fixed-backbone OPDA where at least one IC-POT variant improves over the constant-cost baseline on 24 of the 25 transfers reported in the paper.

\section{Background and Related Work}
\label{sec:background}

Optimal transport compares measures under a pairwise transport cost \citep{peyre2019computational,villani2009optimal}. In the discrete balanced Kantorovich setting, we consider finite supports $X=\{x_i\}_{i=1}^n$, $Y=\{y_j\}_{j=1}^m$, masses $\mu\in\mathbb{R}_+^n$, $\nu\in\mathbb{R}_+^m$, and a transport cost matrix $C\in\mathbb{R}_+^{n\times m}$. The classical problem searches for a coupling $P\in\mathbb{R}_+^{n\times m}$ solving
\begin{equation}
\label{eq:balanced-kantorovich}
\min_{P\ge 0} \ \langle C,P\rangle
\qquad \text{s.t.} \qquad
P\mathbf{1}_m=\mu,
\quad
P^\top \mathbf{1}_n=\nu.
\end{equation}
This full-matching requirement is often too rigid when part of one support should remain unmatched. Classical partial OT addresses this issue by optimizing over sub-couplings. In the Lagrangian formulation of Caffarelli--McCann, the amount of transported mass is governed by a scalar rebate: transporting one more unit is rewarded uniformly, independently of where it comes from or where it goes \citep{caffarelli2010free}. The constant-cost construction of \citet{chapel2020partial} is a discrete version of the same global-control regime. Unbalanced OT relaxes the marginal constraints through global penalties \citep{liero2018optimal,chizat2018scaling}, while semi-relaxed OT allows one-sided marginal relaxation natively \citep{fukunaga2022semiot}. Recent contributions such as mini-batch partial transportation \citep{nguyen2022mpot} and progressive partial OT \citep{zhang2024p2ot} mainly improve optimization or scheduling while keeping the same type of global reject control.

There is also a broader line of work on constrained and sparse OT frameworks. Capacity and other feasibility constraints modify the admissible set of couplings directly \citep{korman2013capacityot,ekren2018constrainedot}, while smooth, sparse, and sparsity-constrained OT aim to structure the transport plan itself through regularization or explicit support restrictions \citep{blondel2018smoothsparse,liu2022sparsityconstrainedot}. %These formulations are important neighbors, but 
Through related to our work, these formulations operate on a different object: they shape feasible couplings or the geometry of the domain, whereas we aim to make unmatched behavior itself dependent on pointwise structure and side information while keeping a partial OT framework.

Several application areas already require control over which samples participate in transport. Guided color or style transfer uses masks, semantic constraints, or user guidance to preserve selected regions \citep{frigo2014colorguided,li2018faceot,kolkin2019strotss}. In domain adaptation, OT has been used to align domains and infer transferable structure from the transport plan itself \citep{courty2017optimal,chang2022uniot}, while masked OT has been proposed in partial domain adaptation because standard OT assumptions can induce biased alignments and negative transfer \citep{luo2023mot}. In selective matching and support subset selection, fully relaxed or globally penalized formulations may suppress points that should remain active in the matching process \citep{riaz2023supportsubset}.

OT has also become a useful tool to account for space-time variabilities in geosience.
%in physical-data problems where displacement matters. 
Among others, Wasserstein distances have been studied for variational geophysical data assimilation to better account for observation-state relationships  \citep{feyeux2018optimal,bocquet2024bridging}; semidiscrete OT has been used for cosmological reconstruction of the early Universe \citep{levy2021fast}; and Lagrangian trajectory assimilation with drifters illustrates the same broader geophysical needs to compare evolving, displaced observations \citep{nilsson2012variational}. The geophysical experiment considered in our work, namely the analysis of satellite-derived ocean wave spectra, also relates to these requirements but targets the unmatched policy itself, that is to say, beyond comparing displaced masses, which parts of two co-located spectra should participate in the comparison as illustrated in Figure~\ref{fig:geophys-real-controlled}.

These works show that OT can already be adapted to selective settings, but typically through a uniform reject mechanism, transport-plan statistics, or task-specific admissibility constraints layered on top of the transport problem. What is still missing in the partial-transport literature, to the best of our knowledge, is a general formulation in which the unmatched policy itself is represented directly in the objective through pointwise costs on the marginals.

\section{IC-POT: Formulation and Structural Properties}

This section presents the IC-POT framework. We first introduce  its slack formulation and the associated reduced Lagrangian problem that positions it relative to classical partial transport. We next derive the dual interpretation and the resulting admissibility structure. We finally present an equivalent augmented-support view, mainly as a well-posedness and implementation statement. Proofs for the formulation and augmented-support results are in Appendix~\ref{app:formulation-proofs}; proofs for the dual and sparsity statements are in Appendix~\ref{app:dual-structure-proofs}.

\subsection{Slack formulation and reduced Lagrangian view}

Using the notations of Section~2, let
\[
  c_s\in\mathbb{R}_+^n,\qquad c_t\in\mathbb{R}_+^m
\]
be pointwise unmatched costs on the source and target supports. IC-POT solves the following constrained OT problem
\begin{equation}
\label{eq:apot-slack}
\begin{aligned}
\min_{P,u,v} \quad &
\langle C, P \rangle + \langle c_s, u \rangle + \langle c_t, v \rangle \\
\text{s.t.}\quad &
P \mathbf{1}_m + u = \mu, \\
&
P^\top \mathbf{1}_n + v = \nu, \\
&
P_{ij} \ge 0, \quad u_i \ge 0, \quad v_j \ge 0.
\end{aligned}
\end{equation}
where $u_i$ and $v_j$ are the unmatched source and target masses. This is the native modeling form of IC-POT: the pairwise cost $C$ still governs matching as in (\ref{eq:balanced-kantorovich}), while $c_s$ and $c_t$ price the local alternative of leaving mass unmatched directly on the original supports.

\begin{proposition}[Generalized Lagrangian partial-transport form]
\label{prop:subcoupling-form}
Problem~\eqref{eq:apot-slack} is equivalent, up to the additive constant
$\langle c_s,\mu\rangle+\langle c_t,\nu\rangle$, to
\[
  \min_{P\in\Gamma_{\le}(\mu,\nu)}
  \sum_{i=1}^n\sum_{j=1}^m
  \bigl(C_{ij}-c_s(i)-c_t(j)\bigr)P_{ij},
\]
where
\[
  \Gamma_{\le}(\mu,\nu)
  =
  \{P\ge 0:\; P\mathbf{1}_m\le \mu,\; P^\top\mathbf{1}_n\le \nu\}.
\]
\end{proposition}

This delivers an entrypoint for the comparison with the Lagrangian formulation of Caffarelli--McCann \citep[Eq.~(1.7)]{caffarelli2010free}. IC-POT replaces the scalar reward for transported mass with the separable rebate $c_s(i)+c_t(j)$, one term on each marginal side. The decomposition preserves the interpretation of $c_s$ and $c_t$ as side-specific rejection costs whose reduced form acts through the pairwise rebate.

\begin{proposition}[Constant-rebate specialization]
\label{prop:constant-rebate-specialization}
If
\[
  c_s(i)=\alpha,\qquad c_t(j)=\beta
\]
for all $i,j$, then Proposition~\ref{prop:subcoupling-form} reduces to a constant-rebate partial-transport objective with uniform rebate $\lambda=\alpha+\beta$.
\end{proposition}
The constant case recovers the global-rebate regime underlying constant-cost partial OT; the partial-W construction of \citet{chapel2020partial} follows after choosing augmented marginals that encode the transported-mass budget. IC-POT makes the rebate pointwise and side-specific, turning unmatched behavior into an explicit modeling variable with local structure.

\subsection{Dual caps and structural consequences}

The reduced Lagrangian form clarifies the mathematical position of IC-POT. The dual then shows how the unmatched costs act inside the optimization itself.

\begin{proposition}[Dual formulation]
\label{prop:dual}
The dual of Problem~\eqref{eq:apot-slack} is
\begin{equation}
\label{eq:apot-dual}
\begin{alignedat}{2}
\max_{f,g}\quad & \sum_{i=1}^n f_i\mu_i+\sum_{j=1}^m g_j\nu_j \\
\text{s.t.}\quad & f_i+g_j\le C_{ij}, &\quad& i=1,\dots,n,\; j=1,\dots,m,\\
& f_i\le c_s(i), &\quad& i=1,\dots,n,\\
& g_j\le c_t(j), &\quad& j=1,\dots,m.
\end{alignedat}
\end{equation}
\end{proposition}

The caps $f_i\le c_s(i)$ and $g_j\le c_t(j)$ make unmatched costs local acceptance thresholds: small values make rejection cheap, while larger values protect a point under a wider range of transport costs.

\begin{proposition}[Complementary slackness]
\label{prop:cs}
Let $(P,u,v)$ and $(f,g)$ be primal-dual optimal. Then
\[
P_{ij} > 0 \;\Longrightarrow\; f_i + g_j = C_{ij},
\]
\[
u_i > 0 \;\Longrightarrow\; f_i = c_s(i),
\qquad
v_j > 0 \;\Longrightarrow\; g_j = c_t(j).
\]
\end{proposition}

Accepted transport therefore saturates the pairwise constraint, whereas rejected mass saturates the side-specific cap on the corresponding support point.

\begin{proposition}[Admissible support]
\label{prop:dominated}
If $C_{ij} > c_s(i) + c_t(j)$, then every optimal solution of \eqref{eq:apot-slack} satisfies $P_{ij}=0$. If $C_{ij} = c_s(i) + c_t(j)$, then there exists an optimal solution with $P_{ij}=0$.
\end{proposition}

An edge can therefore be active only when its transport cost is no larger than the sum of the local unmatched costs. This exact test is the main sparsity consequence of the model, and it yields the support restriction used in our solver, proved equivalent to the full LP in Appendix~\ref{app:optimization}.
A two-source one-target counterexample in Appendix~\ref{app:dual-structure-proofs} shows strict separation from constant-cost partial OT: pointwise unmatched costs can break symmetries that a uniform residual rule preserves.

In practice, the unmatched functions are used in two regimes. They may specify participation, by making some regions available for comparison and others cheap to reject; this is the role played by physical reliability priors in the geophysical experiment. They may also modulate rejection pressure on a fixed support, by protecting points according to confidence, sampling bias, or local geometry as illustrated in PU learning and OPDA. Appendix~\ref{app:intent-instantiation} gives the corresponding instantiation details.

\subsection{Augmented-support equivalence and implementation view}

%The last equivalence is mainly technical. 
Although IC-POT is introduced through side-specific slack variables, it can be rewritten as a standard balanced Kantorovich OT problem once one augments each support with a dummy point.

\begin{proposition}[Augmented-support equivalence]
\label{prop:augmented-support}
Eq.~\eqref{eq:apot-slack} is equivalent to balanced transport with
\[
\bar\mu=(\mu,\|\nu\|_1),\qquad
\bar\nu=(\nu,\|\mu\|_1),\qquad
\bar C=
\begin{pmatrix}
C & c_s\\
c_t^\top & 0
\end{pmatrix}
\]
\end{proposition}

This dummy construction proves well-posedness within the usual discrete Kantorovich framework and provides the implementation view used by our LP solver. The original slack formulation remains the modeling object: the specification consists of the $n+m$ rejection costs on the original supports, and the augmented cost matrix is the corresponding balanced-OT representation.
We provide a lightweight Python Library implementation of IC-POT with its sparse solver at \url{https://anonymous.4open.science/r/IC-POT-68F4/}
\section{Experiments}

We report numerical experiments to highlight the relevance and genericity of IC-POT compared with constant-cost partial OT.  Section~4.1 studies IC-POT as a local rejection mechanism: in PU learning and OPDA, the task is already a reject-structured problem, and the question is whether pointwise unmatched costs exploit sampling bias, confidence, or latent support geometry better. %than a uniform partial-transport rule.
Section 4.2 involves a geophysical case study where IC-POT can truly encode physical priors defining what should participate in the transport problem before the comparison is solved.

\subsection{Structured Rejection: From Controlled Regimes to Realistic Benchmarks}

\begin{figure}[!t]
\centering
\includegraphics[width=0.93\linewidth]{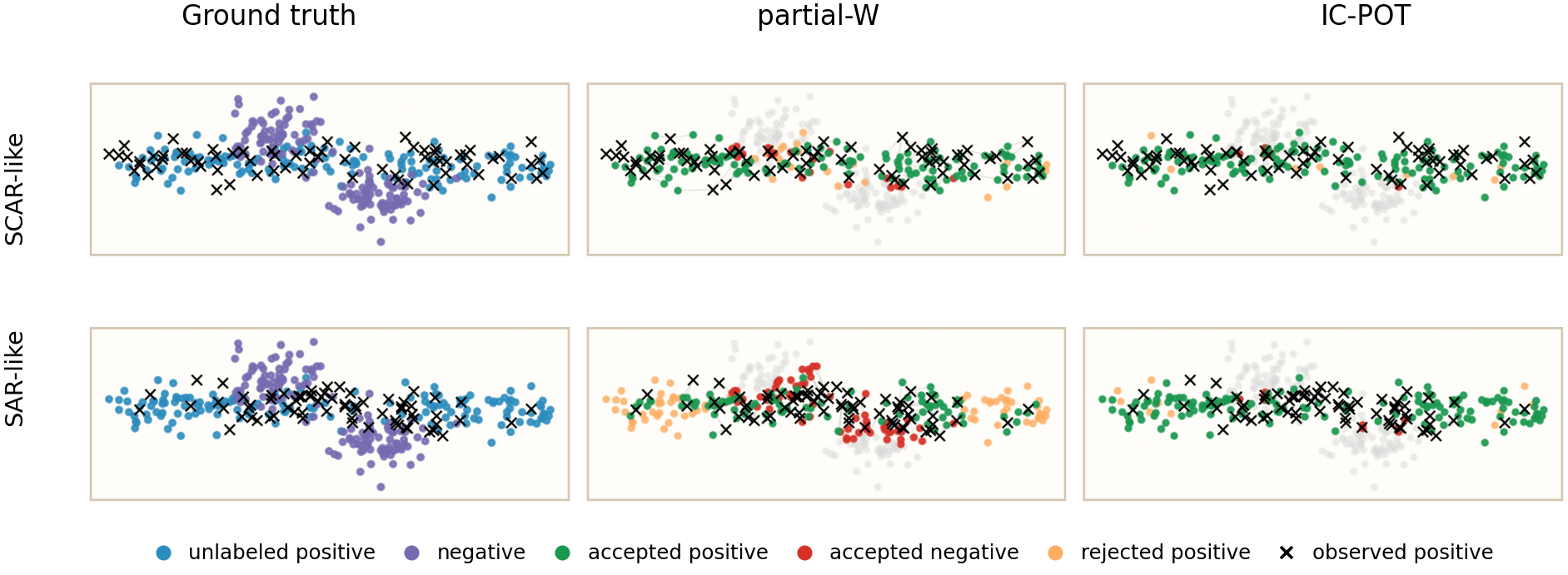}
\caption[Controlled PU selection-bias test]{Controlled PU selection-bias test on a representative seed. The top row is homogeneous / SCAR-like; the bottom row is heterogeneous / SAR-like in the PU sense. Columns compare the ground truth, the constant-cost \emph{partial-W} rule, and the aligned IC-POT profile used in the text. Black crosses denote observed positives, and thin segments indicate transported mass. Full setup and unmatched-function definitions are given in Appendix~\ref{app:pu-selection-bias-extra}.}
\label{fig:pu-toy}
\end{figure}

We begin with tasks where rejection is already part of the learning problem. Here, IC-POT does not change the representation or the transport geometry; it changes how rejection is priced. PU learning isolates the mechanism when the source of structure is known {\em a priori}. OPDA tests the same idea when the structure must be inferred from confidence and local support in a learned representation.

\subsubsection{PU Learning: Controlled Structured Rejection}
\label{subsec:pu-testbed}

Positive-unlabeled (PU) learning observes labeled positives and an unlabeled pool mixing latent positives and negatives. It is a direct testbed for partial transport: the learner must keep the positive-compatible part of the unlabeled distribution and reject the rest. We use it to test a precise claim in a setting where constant-cost partial OT is already relevant \citep{chapel2020partial}: a structured unmatched policy should help when the selection mechanism is itself structured. Under \emph{selected completely at random} (SCAR), observed positives remain representative of the latent positive support, so a global rejection rule is already aligned with the problem. Under \emph{selected at random} (SAR, in the PU-learning sense), selection depends on covariates, so some positive regions are under-observed even though they should remain compatible with the positive class. The desired reject rule is then no longer uniform: low observed support can mean sampling bias, not negative evidence.

The construction keeps the latent geometry fixed across regimes. Positives lie on a horizontal band, negatives form two lateral modes, and only the positive-selection mechanism changes: homogeneous / SCAR-like selection is nearly uniform, while heterogeneous / SAR-like selection concentrates near the center and under-represents the positive fringes. We compare two policies under the same transported-mass prior $s=\pi$. The baseline is the constant-cost \emph{partial-W} rule. IC-POT keeps the same transport geometry and global prior, and changes only the source-side unmatched cost. Its profile follows the selection mechanism: in the heterogeneous regime, lateral positive regions have low observed support for selection reasons, so discarding them should be expensive. This makes the unmatched policy a direct encoding of the known sampling mechanism, while the target-side cost remains uniform because the controlled bias concerns the observed-positive support. The exact setup, especially the parametrization of $c_s$ and $c_t$ are given in Appendix~\ref{app:pu-selection-bias-extra}.

Figure~\ref{fig:pu-toy} shows the expected transition. The gap is modest in the homogeneous regime and large in the heterogeneous one: the constant-cost rule accepts more incompatible unlabeled points and misses positive fringe mass, while IC-POT remains selective near the center and preserves the fringes. Over 5 seeds, \emph{partial-W} reaches mean F1 $0.82 \pm 0.03$ / $0.52 \pm 0.06$ and mean accuracy $0.81 \pm 0.03$ / $0.50 \pm 0.06$ in the homogeneous / heterogeneous regimes; IC-POT reaches $0.87 \pm 0.02$ / $0.86 \pm 0.01$ in F1 and $0.86 \pm 0.02$ / $0.85 \pm 0.01$ in accuracy. The gain comes from matching the reject rule to the selection mechanism, not from changing the transport geometry. Appendix~\ref{app:pu-selection-bias-extra} reports selection-bias sweeps, negative-geometry sensitivity, hyperparameter checks, and a misaligned control.

\subsubsection[OPDA]{OPDA: Structured Rejection from Data Geometry}
\label{subsec:opda-benchmark}

We here, move from the controlled PU setting to a realistic benchmark in the same family of reject-structured problems. In open-partial domain adaptation (OPDA), some source classes are absent from the target domain and the target domain may contain unknown classes. The reject decision must preserve transfer to shared classes while discarding target-private samples. OT-based alignment is already standard in domain adaptation \citep{courty2017optimal}, universal variants have been formulated in OT terms \citep{chang2022uniot}, and partial domain adaptation has been addressed with masked OT \citep{luo2023mot}. OPDA therefore tests whether a pointwise unmatched policy can improve the final reject layer under a fixed representation and transport geometry.

We keep the UniOOD representation, source prototypes, and calibration pipeline fixed, starting from the calibrated CLIP baseline of \citet{zhang2024clip} (\emph{CLIP dis}), and modify only the final reject layer. This isolates the unmatched-cost design: all variants use the same features, prototypes, calibration, and transport geometry. The uniform \emph{partial-W} baseline is the constant-cost policy introduced above, $c_s(i)=c_t(j)=A$. The two IC-POT variants encode unsupervised evidence that a target sample lies on the shared support. Entropy-based pricing protects samples with low posterior entropy, because the fixed representation already assigns them clearly to a source prototype. Prototype-support pricing protects samples whose neighbors support the same prototype, because local agreement is stronger evidence of shared-class structure than isolated confidence. These are structural diagnostics available at test time from the calibrated backend, not labels or oracle signals.

In all OPDA experiments we keep $c_t(j)=A$. The second marginal is a stable prototype set, while the relevant uncertainty is target-sample side: ambiguous, isolated, or locally unsupported target samples should be easy to reject, whereas confident or locally coherent samples should be protected. Appendix~\ref{app:opda-unmatched} gives the exact functions and hyperparameters, selected on two representative Office-Home transfers and then kept fixed across the full-table evaluation. Appendix~\ref{app:structure-metrics} reports unsupervised diagnostics showing that Office-31 and VisDA have clearer local organization, Office-Home is more diffuse, and DomainNet is mixed. This matters because the two IC-POT policies should not be expected to behave identically across datasets: entropy is more robust when neighborhoods are diffuse, while prototype-support is meaningful when local neighborhoods are coherent.

Table~\ref{tab:clipdistill-comparison} reports the resulting reject layers together with representative OPDA baselines from the CLIP distillation benchmark. All transport variants improve substantially over the same fixed CLIP distillation backbone, isolating the gain in the reject policy. The IC-POT variants also remain competitive with specialized OPDA methods while changing only the final decision layer \citep{lee2025tlsa}. A local entropy-threshold baseline stays below the transport variants, indicating that the gain comes from transport-coupled rejection rather than standalone confidence filtering. Across the 24 Office-31, Office-Home, and DomainNet transfers in Appendix~\ref{app:opda-per-transfer}, at least one IC-POT variant improves over \emph{partial-W} on 23 transfers; VisDA follows the same trend, giving 24/25 overall. The entropy-based variant alone improves over \emph{partial-W} on 19 of these 25 transfers.

\begin{table}[tb]
\centering
\caption[OPDA results in the CLIP distillation benchmark]{OPDA results in the standard CLIP distillation (CLIP dis) benchmark, reported in H-score and H$^3$-score. We include source-only (SO) and representative baselines reported in Table~5 of \citet{zhang2024clip}: DANCE \citep{saito2020dance}, OVANet \citep{saito2021ovanet}, UniOT \citep{chang2022uniot}, and CLIP dis \citep{zhang2024clip}. We also report a local entropy-threshold baseline, the constant-cost partial-transport baseline (\emph{partial-W}), and two IC-POT variants.}
\label{tab:clipdistill-comparison}
\small
\resizebox{\linewidth}{!}{
\begin{tabular}{lcccccccc}
\toprule
& \multicolumn{2}{c}{Office-31} & \multicolumn{2}{c}{Office-Home} & \multicolumn{2}{c}{VisDA} & \multicolumn{2}{c}{DomainNet} \\
\cmidrule(lr){2-3}\cmidrule(lr){4-5}\cmidrule(lr){6-7}\cmidrule(lr){8-9}
Method & H & H$^3$ & H & H$^3$ & H & H$^3$ & H & H$^3$ \\
\midrule
SO & 91.98 & 89.95 & 84.52 & 82.70 & 69.85 & 74.24 & 61.49 & 64.65 \\
DANCE &94.70 &\textbf{91.64} & 89.01 & 85.60 & 71.90 & 75.77 & 60.53 & 63.94 \\
OVANet & 93.36 & 90.80 & 85.42 & 83.33 & 59.47 & 66.07 & 70.70 & 71.10 \\
UniOT & 92.32 & 89.07 & 89.45 & 87.09 & 79.10 & 77.69 & 71.42 & 69.90 \\
CLIP dis & 86.74 & 86.45 & 86.40 & 83.73 & 84.74 & 84.80 & 72.37 & 72.08 \\
\midrule
CLIP dis + entropy-only & 90.44 & 88.62 & 88.12 & 84.91 & 81.02 & 82.29 & 77.39 & 75.52 \\
CLIP dis + partial-W & 94.08 & 90.85 & 93.17 & 87.97 & 86.08 & 85.69 & 79.20 &  76.70 \\
CLIP dis + IC-POT (entropy-based) & 94.23 & 90.94 &\textbf{93.23} &\textbf{88.00} & 86.31 & 85.84 & 79.70 & \textbf{77.01} \\
CLIP dis + IC-POT (prototype-support) & \textbf{95.12} & 91.54 & 91.78 & 87.06 & \textbf{87.28} & \textbf{86.48} & \textbf{79.71} & \textbf{77.01} \\
\bottomrule
\end{tabular}
}
\end{table}

The dataset-level pattern matches the structural motivation. Prototype-support gives the strongest gains in locally coherent regimes, reaching $95.12$ H on Office-31 and $87.28$ on VisDA, more than one point above \emph{partial-W}. Office-Home is more diffuse: \emph{partial-W} already captures most of the improvement, entropy-based pricing is nearly neutral, and prototype-support is less stable. DomainNet is intermediate, with small gains from both IC-POT variants and little separation between them. Thus the OPDA results support the same conclusion as the PU testbed in a realistic setting: once representation and calibration are fixed, part of the remaining performance depends on whether rejection is modeled as a uniform scalar rule or as a structure-dependent policy.

\subsection{Encoding geophysical priors: Structured Comparability in SWIM/SAR Wave Spectra}
\label{subsec:wave-spectrum-modeling}

This experiment addresses the analysis of geophysical data to test the case where the unmatched policy helps to define the comparison itself. The data are directional ocean-wave spectra retrieved from SWIM, the wave scatterometer onboard CFOSAT \citep{hauser2021swim}, and Synthetic Aperture Radar (SAR), a classical source of ocean-wave spectral information \citep{hasselmann1985sar}. Each ocean-wave spectrum distributes energy over wavenumber $k$ and propagation direction $\phi$; localized peaks correspond to wave systems such as swell or wind sea. The target is the comparable part of two retrieved ocean wave spectra: systems that can be related across observation chains with different transfer functions, artifacts, and sampling conditions.

  \begin{figure}[t]
    \centering
    \begin{minipage}{0.85\linewidth}
        \centering
        \includegraphics[width=\linewidth]{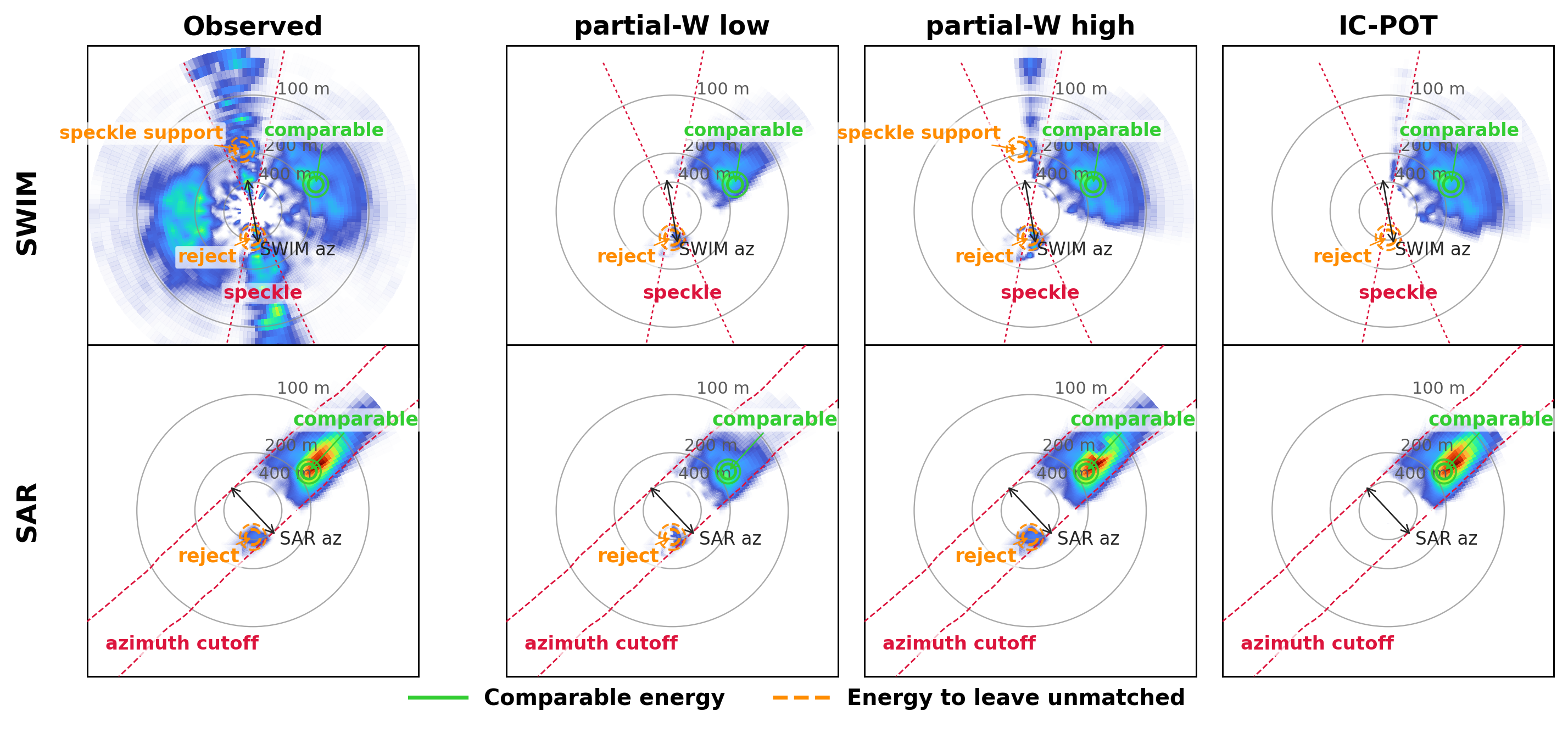}\\[-0.3em]
        {\small \textbf{(a)} Real SWIM-SAR co-location.}
    \end{minipage}

    \vspace{-0.2em}

    \begin{minipage}{\linewidth}
        \centering
        \includegraphics[width=0.85\linewidth]{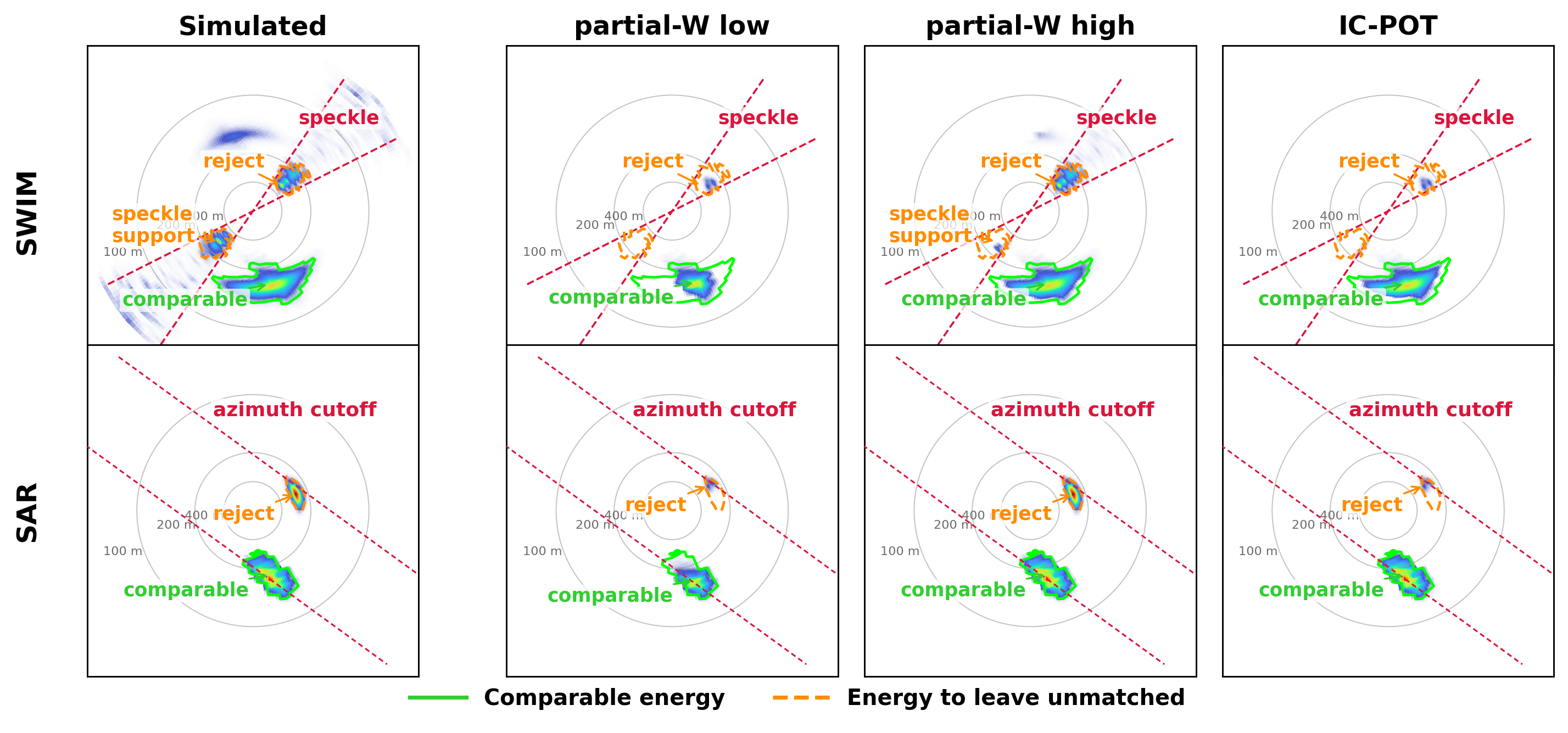}\\[-0.3em]
        {\small \textbf{(b)} Simulated SWIM-SAR case with known comparable and unmatched components.}
    \end{minipage}

    \vspace{-0.4em}

    \caption{Illustration of partial OT for SWIM-SAR ocean-wave spectra. (a) Real co-location: SWIM/SAR observed spectra and
  matched supports. (b) Simulated case: sensor-guided spectra and SAR-side mass matched to SWIM by two constant-cost
  \emph{partial-W} rules and by IC-POT. Green marks comparable wave energy; orange marks components to leave unmatched.}
    \label{fig:geophys-real-controlled}
  \end{figure}

Figure~\ref{fig:geophys-real-controlled}(a) depicts real co-located SWIM-SAR ocean-wave spectra. Both sensors provide ocean-wave spectra retrieved through sensor-specific transfer functions, acquisitions may be separated in space and time with a spatial and temporal delta, respectively, under 100km and 6h. SWIM derived wave spectra is quasi-symmetric (two-sided) and can be corrupted by speckle-dominated directional sectors and an axial ambiguity (angular sector within red dashed lines) ; whereas SAR derived wave spectra are impacted by azimuth cutoff, which displace energy in the spectral plane along azimuth direction\citep{hasselmann1985sar,stopa2015azimuthcutoff}, within the spectra area between the two dashed "azimuth cutoff" lines. These perturbating effects play different roles: speckle-dominated SWIM energy gives weak local evidence, cutoff-affected SAR energy can remain coherent after displacement, and one-sided SAR components with weak SWIM support naturally enter the unmatched channel.

IC-POT makes this comparability rule explicit. Because the two sensors carry different sensing modes, the SAR-side and SWIM-side unmatched costs encode different information: $c_s$ protects cutoff-displaced SAR mass when local SWIM support makes it comparable, while exposing one-sided or unsupported SAR components to rejection; $c_t$ protects reliable SWIM support and makes speckle-dominated sectors inexpensive to leave unmatched. The transport geometry is kept fixed; the prior acts only through local rejection prices. appendix~\ref{app:geophysical-motivation} details the parameterization of $c_s$ and $c_t$.

Figure~\ref{fig:geophys-real-controlled}(a) qualitatively compares IC-POT with constant-cost POT on a real co-location. In the displayed SWIM-SAR pair, the green-marked systems are separated in the spectral plane and have different directional spreads, but they correspond to energy from the same physical wave system. On the SAR side, this energy has been displaced by sensor effects and folded below the azimuth cutoff; it should therefore remain part of the comparison despite its geometric offset from the SWIM support. The orange-marked component has a different origin: it reflects a SAR directional-ambiguity artifact and lies close, in the spectral plane, to SWIM energy dominated by azimuthal speckle. Matching this component would create an artificial correspondence with unreliable SWIM support, rather than a comparison between physically consistent wave systems. The SWIM-side support row highlights the complementary effect: when the reliable SWIM support carries less energy than the corresponding SAR component, a larger matched mass can draw on nearby speckle-supported SWIM energy.

To quantitatively benchmark the considered POT schemes, we use 100 realistic synthetic SWIM-SAR pairs with sensor-specific effects and known comparable support, illustrated in Figure~\ref{fig:geophys-real-controlled}(b). We report normalized comparable recovery, unmatch precision, reliable loss, and spurious transport scores, using the same $c_s,c_t$ parameterization as for real SWIM/SAR data; details are in Appendix~\ref{app:geophysical-motivation}.

%Real co-locations make the intended comparison visible, while instrument-guided synthetic spectra provide known comparable support for quantitative evaluation. Each synthetic pair contains common wave systems, SAR cutoff displacement, SWIM speckle sectors, and one-sided components. The unmatched costs are built only from observable quantities: local energy support, an azimuth-cutoff prior on the SAR side, and a speckle-sector prior on the SWIM side. The cutoff prior protects coherent SAR energy that may have been displaced by the acquisition process when compatible SWIM support is present; the speckle prior makes SWIM-supported matches unreliable inside noisy directional sectors. We compare IC-POT to the \emph{partial-W} family of \citet{chapel2020partial}, i.e., the constant-cost specialization $c_s=c_t=A$, on the same grid and native relative masses. Metrics report comparable recovery, unmatch precision, reliable loss, and spurious transport; Appendix~\ref{app:geophysical-motivation} gives the cost definitions, scalar sweep, and additional cases.

\begin{table}[t]
  \centering
  \caption{Partial OT performance metrics for the synthetic SWIM/SAR ocean-wave spectra dataset: for each metric, we report
  mean $\pm$ standard deviation values over the 100 samples of the dataset. We refer the reader to the main text for the
  definition of the considered normalized evaluation metrics.}
  \label{tab:geophys-controlled}
  \small
  \resizebox{\linewidth}{!}{
  \begin{tabular}{lcccc}
  \toprule
  Method & Comparable recovery $\uparrow$ & Unmatch precision $\uparrow$ & Reliable loss $\downarrow$ & Spurious transport $
  \downarrow$ \\
  \midrule
  \emph{partial-W} low & $0.023 \pm 0.041$ & $0.691 \pm 0.047$ & $0.070 \pm 0.017$ & $\mathbf{0.019 \pm 0.005}$ \\
  \emph{partial-W} high & $\mathbf{0.999 \pm 0.013}$ & $0.761 \pm 0.045$ & $\mathbf{0.000 \pm 0.001}$ & $0.236 \pm 0.027$ \\
  IC-POT & $0.993 \pm 0.013$ & $\mathbf{0.770 \pm 0.044}$ & $\mathbf{0.000 \pm 0.001}$ & $0.031 \pm 0.09$ \\
  \bottomrule
  \end{tabular}
  }
  \end{table}

Figure~\ref{fig:geophys-real-controlled}(b) and Table~\ref{tab:geophys-controlled} summarize our evaluation for the synthetic dataset. The simulated comparisons reproduce the decision structure illustrated in Figure~\ref{fig:geophys-real-controlled}(a): the comparable system is displaced in the SAR spectrum by sensor effects and folded below the azimuth cutoff, while the one-sided SAR component has no reliable SWIM counterpart and lies near speckle-supported SWIM energy. Quantitatively, IC-POT keeps the comparable recovery of the high-price partial-W regime ($0.993$ vs. $0.999$)  while reducing spurious transport by more than a factor of 7 ($0.031$ vs. $0.236$), and it remains close to the low-price regime in spurious transport while recovering almost all comparable mass. A low partial-W price leaves comparable wave energy aside; a high price recovers the common system and also transports one-sided or speckle-supported mass. IC-POT separates these decisions because rejection is priced pointwise and differently on the two supports: cutoff-compatible SAR energy is protected when supported by SWIM, while speckle-supported or locally unsupported components remain inexpensive to reject. The full scalar sweep in Appendix~\ref{fig:geophys-appendix-tradeoff} shows the same pattern over the constant-cost family. The gain is therefore not only numerical. %It is also an expressibility result: 
IC-POT turns the intended SWIM/SAR comparison into a native partial-transport objective under fixed geometry, whereas a scalar reject rule can only move along a one-dimensional mass-rejection trade-off.

\section{Discussion and Limitations}

IC-POT changes what is specified in a partial-transport model. The transport cost $C$ defines the matching geometry, while $c_s$ and $c_t$ define how mass may leave the comparison. In the experiments these unmatched costs are prescribed from observable side information: a covariate-aligned selection proxy for PU learning, posterior confidence and neighborhood support for OPDA, and sensor-derived reliability maps for SWIM/SAR. Thus $c_s$ and $c_t$ are part of the model specification rather than a numerical tuning detail. Appendix~\ref{app:intent-instantiation} summarizes these instantiations, and Appendix~\ref{app:selective-color-transfer} gives a visual selective-transfer example.

This design is useful when the intended reject rule has structure. In PU learning and OPDA, the structure comes from sampling bias, confidence, and local geometry. In wave-spectrum comparison, the same mechanism has a more fundamental role: the unmatched profiles specify the reliable overlap to be compared before the transport plan is computed. This is precisely the role that a scalar partial-transport rule cannot express. The synthetic SWIM/SAR cases provide known comparable support for quantitative evaluation, while real co-locations show that the same profiles correspond to interpretable reliability patterns in observed spectra.

Unmatched policies can also be diagnosed without task labels. In OPDA, posterior entropy, margins, and neighborhood consistency indicate whether confidence-based or support-based rejection should be more reliable. In the geophysical case, cutoff estimates, speckle sectors, and local spectral support play the same role. These diagnostics allow one to check whether the protected and rejected regions match the intended support structure before downstream evaluation.

The main limitation is that the unmatched functions are specified by the user. This is appropriate when diagnostics or structural proxies are available, and suggests learning $c_s$ and $c_t$ from validation labels, weak supervision \citep{ratner2020snorkel}, or bilevel criteria\citep{franceschi2018bilevel,amos2017optnet}.Scalability is another limitation: larger applications will require admissible-edge pruning from dual caps, warm starts across reject profiles, or scalable one-dimensional, sliced, and mini-batch partial-transport approximations \citep{chapel2025pawl,nguyen2022mpot}. Standard entropic regularization is not a drop-in replacement; see Appendix~\ref{app:entropy-limitations}.

\section{Conclusion}

IC-POT extends partial OT by making unmatched behavior explicit, pointwise, and side-specific on the two marginals. Eliminating the slack variables connects the formulation to a separable-rebate generalization of Lagrangian partial transport, while the dual and admissibility results show how unmatched costs act inside the optimization. Empirically, the same mechanism supports structured rejection in PU learning and OPDA, and physically motivated comparability in SWIM/SAR wave-spectrum retrieval. The main contribution is therefore a targeted generalization of partial transport in which the reject policy becomes an explicit object of the formulation: it can be specified from side information, inspected before solving, and evaluated through the resulting transport plan.

\clearpage  
\bibliographystyle{plainnat}
\bibliography{references}
\clearpage
\appendix
\section*{Appendix organization}
The appendix follows the narrative of the paper. We first give the experimental material that supports the main modeling claim: the geophysical unmatched profiles, metrics, trade-off plots, and controlled cases, followed by the selective color-transfer illustration of expressivity. We then provide the mathematical proofs and optimization details. The final appendices give the PU and OPDA experimental details, structural diagnostics, per-transfer tables, and the discussion of entropic regularization.

\section{Geophysical motivation, sensor priors, and synthetic cases}
\label{app:geophysical-motivation}

The geophysical experiment in Section~\ref{subsec:wave-spectrum-modeling} is motivated by comparison problems in which two retrieved observations carry side-dependent reliability information. Ocean wave-spectrum retrieval gives a concrete instance. A directional spectrum assigns wave energy to wavenumber $k$ and propagation direction $\phi$; localized packets correspond to physical wave systems. A SWIM spectrum and a SAR spectrum can contain the same swell system alongside different retrieval artifacts. SWIM may contain speckle-dominated directional sectors and an axial ambiguity, while SAR is affected by azimuth cutoff and velocity-bunching effects that can displace energy in the spectral plane \citep{hauser2021swim,hasselmann1985sar,stopa2015azimuthcutoff}.

The intended comparison is therefore a comparability statement. Some spectral energy should remain available for transport because it represents the same wave system across the two retrievals. Some energy should be left unmatched because it is one-sided or supported by an unreliable sector alone. Some SAR energy should also be allowed to participate despite a larger spectral displacement when the displacement is consistent with cutoff or velocity bunching and local SWIM support is present. IC-POT expresses these decisions through side-specific unmatched costs: physically supported cutoff-displaced mass is protected from rejection, while unsupported or speckle-dominated mass remains inexpensive to leave unmatched.

\paragraph{Synthetic construction.}
The quantitative experiment uses synthetic spectra because the comparable support is then known by construction. Each case is generated on a north-referenced $(k,\phi)$ grid. Latent wave systems are represented by localized directional spectra with classical spreading functions. The SWIM observation applies axial ambiguity and adds speckle energy inside an azimuthal sector. The SAR observation applies an azimuth cutoff and a velocity-bunching/cutoff operator that redistributes part of the energy toward the admissible side of the cutoff. We also add one-sided systems on either side. The solver is run from SAR to SWIM, using the native relative masses of the two spectra.

For each synthetic pair we keep diagnostic maps on the same grid. The map
$b_{\mathrm{SAR}}(z)\in[0,1]$ marks SAR energy affected by cutoff or velocity-bunching displacement, $s_{\mathrm{SWIM}}(z)\in[0,1]$ marks the SWIM speckle sector, and $a_{\mathrm{SAR}}(z),a_{\mathrm{SWIM}}(z)\in[0,1]$ denote normalized observed spectral-energy maps. These maps are available from the simulator in the synthetic benchmark; in real co-locations they correspond to sensor diagnostics such as speckle maps, cutoff estimates, and local spectral support.

\paragraph{Local support.}
The cost profiles use a local-support operator to make decisions coherent at the scale of a wave system. Let $\mathcal S_{\rho_{\mathrm{loc}}}(E)$ be a normalized support map obtained by spreading the positive energy of $E$ inside a ball of radius $\rho_{\mathrm{loc}}$ in the $(\log k,\phi)$ grid. We use
\[
q_{\mathrm{SWIM}}(z)=
\mathcal S_{\rho_{\mathrm{loc}}}\!\left(a_{\mathrm{SWIM}}(1-s_{\mathrm{SWIM}})^{p_{\mathrm{sp}}}\right)(z),
\qquad
q_{\mathrm{SAR}}(z)=
\mathcal S_{\rho_{\mathrm{loc}}}\!\left(a_{\mathrm{SAR}}(\eta_{\mathrm{cut}}+(1-\eta_{\mathrm{cut}})b_{\mathrm{SAR}})\right)(z).
\]
We also compute $h_{\mathrm{sp}}(z)=\mathcal S_{\rho_{\mathrm{sp}}}(a_{\mathrm{SWIM}}s_{\mathrm{SWIM}})(z)$, which measures whether a SAR location is supported primarily by a SWIM speckle sector. Thus ordinary SAR energy is protected when reliable SWIM support is nearby and not explained only by speckle. Cutoff-affected SAR energy receives a direct protection branch, which encodes the case where a larger spectral displacement remains physically plausible. SWIM energy is protected when it is reliable and supported by coherent SAR energy.

\paragraph{Unmatched costs.}
The source-side unmatched cost is the SAR profile and the target-side unmatched cost is the SWIM profile. In the synthetic benchmark we first define normalized protection scores
\[
r_s(z)=
\left[
\alpha_{\mathrm{SAR}}(z)
\left(
\beta_{\mathrm{cut}}b_{\mathrm{SAR}}(z)^{\gamma_{\mathrm{cut}}}
\;+\;
\beta_{\mathrm{loc}}q_{\mathrm{SWIM}}(z)(1-h_{\mathrm{sp}}(z))^{p_{\mathrm{veto}}}
\right)
\right]_{\times},
\]
\[
r_t(z)=
\left[
\alpha_{\mathrm{SWIM}}(z)q_{\mathrm{SAR}}(z)
\left(1-s_{\mathrm{SWIM}}(z)\right)^{p_{\mathrm{sp}}}
\right]_{\times},
\]
where $\alpha_{\mathrm{SAR}}=a_{\mathrm{SAR}}^{0.35}$, $\alpha_{\mathrm{SWIM}}=a_{\mathrm{SWIM}}^{0.45}$ and $[\cdot]_{\times}$ denotes clipping to $[0,1]$. The unmatched costs are then obtained by the same affine map used throughout the paper,
\[
c_{\mathrm{SAR}}(z)=c_{\min}+(c_{\max}-c_{\min})r_s(z),
\qquad
c_{\mathrm{SWIM}}(z)=c_{\min}+(c_{\max}-c_{\min})r_t(z).
\]
The first score protects observed SAR energy through two observable mechanisms: a direct cutoff/VB branch for displaced mass, and a weaker local-support branch vetoed by nearby SWIM speckle. Unsupported SAR packets therefore remain cheap to reject. The second score protects observed SWIM energy only when SAR support is nearby and the SWIM sector is reliable; speckle-dominated mass remains inexpensive to leave unmatched. The transport cost itself remains the fixed cyclic Euclidean cost in $(\log k,\phi)$ for all methods. The reported runs use a fixed profile for all cases, with $\rho_{\mathrm{loc}}=0.16$, $\rho_{\mathrm{sp}}=0.20$, $p_{\mathrm{sp}}=7$, $\gamma_{\mathrm{cut}}=0.75$, $p_{\mathrm{veto}}=2.5$, $\beta_{\mathrm{cut}}=2.5$, $\beta_{\mathrm{loc}}=0.15$, $c_{\min}=5\cdot10^{-5}$, and $c_{\max}=3.0$. These constants are selected once and kept fixed across the benchmark.

\paragraph{Metrics.}
The synthetic construction gives binary masks for comparable and non-comparable supports on both sides. The reported metrics are computed for the SAR$\to$SWIM transport direction. Let $m_j^{\mathrm{SAR}}$ denote the matched mass leaving each SAR grid cell, $u_j^{\mathrm{SAR}}$ the unmatched SAR mass, and $u_i^{\mathrm{SWIM}}$ the unmatched SWIM mass. Comparable recovery is the fraction of SAR-side ground-truth comparable mass that is transported. Unmatch precision is the fraction of unmatched mass that lies in the ground-truth non-comparable support. Reliable loss is the SAR-side comparable mass left unmatched. Spurious transport is computed from the transport plan as the mass incident to SAR or SWIM cells marked non-comparable, with the overlap counted once. The first two metrics should be high, while the last two should be low.

\begin{figure}[p]
\centering
\includegraphics[width=\linewidth]{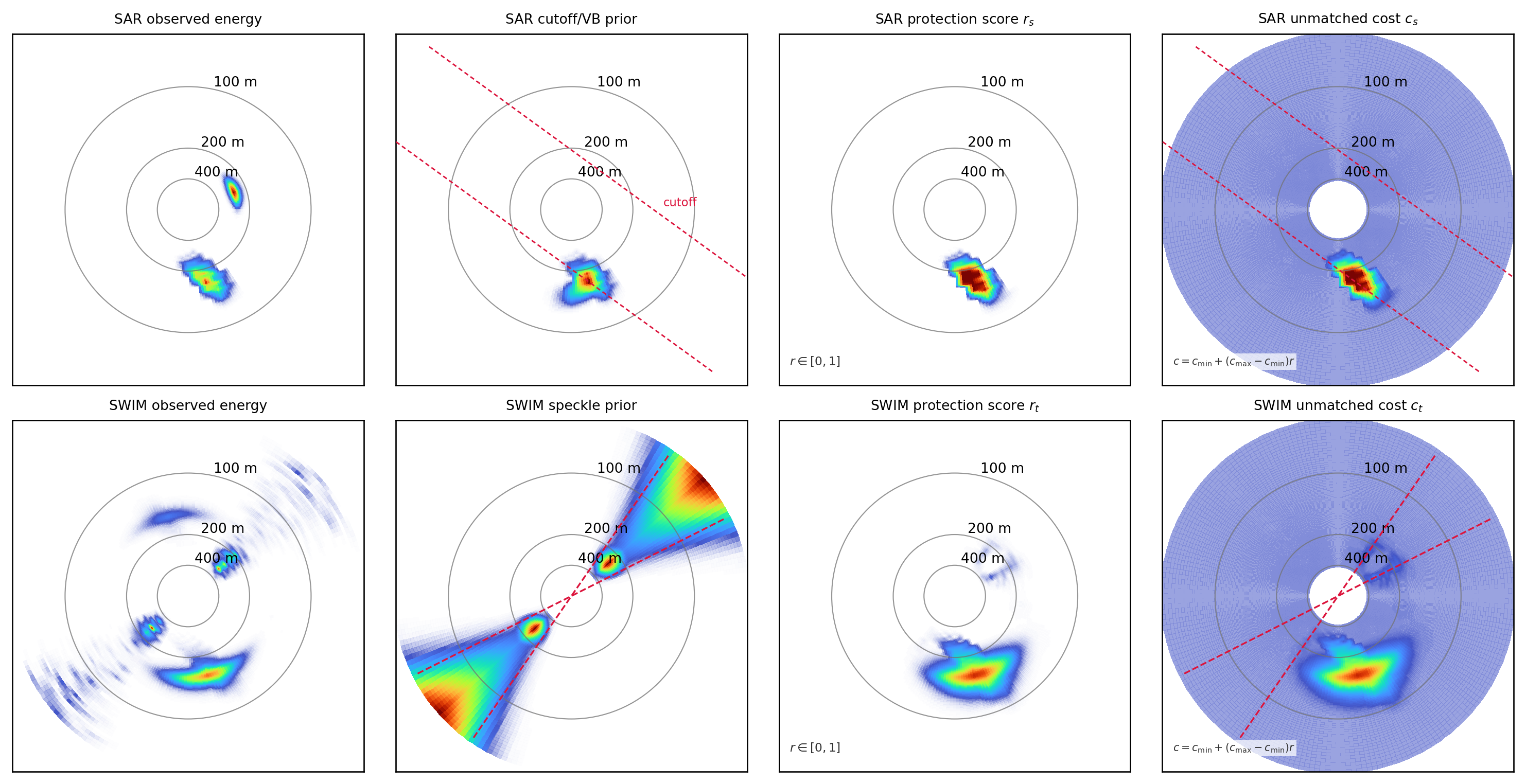}
\caption{Construction of the side-specific unmatched costs for the synthetic SWIM/SAR case used in Figure~\ref{fig:geophys-real-controlled}(b). The first row builds the SAR-side protection score $r_s$ from observed SAR energy, the cutoff/velocity-bunching prior, and reliable local SWIM support. The second row builds the SWIM-side protection score $r_t$ from observed SWIM energy, the speckle prior, and reliable local SAR support. The final costs are $c_{\min}+(c_{\max}-c_{\min})r_s$ and $c_{\min}+(c_{\max}-c_{\min})r_t$: low cost makes rejection inexpensive, while high cost protects mass from rejection. Red dashed curves mark sensor priors.}
\label{fig:geophys-appendix-costs}
\end{figure}

\paragraph{Partial-W trade-off.}
The constant-cost baseline uses a single unmatched price for all grid cells on both sides. Varying this price traces a one-dimensional trade-off. Low prices leave many artifacts unmatched and also leave comparable wave energy aside. High prices recover more comparable energy and also transport speckle-supported or one-sided mass. Figure~\ref{fig:geophys-appendix-tradeoff} shows this sweep for a representative synthetic case. IC-POT uses side-specific unmatched profiles under the same transport geometry, placing it away from the scalar partial-W curve.

\begin{figure}[p]
\centering
\includegraphics[width=0.72\linewidth]{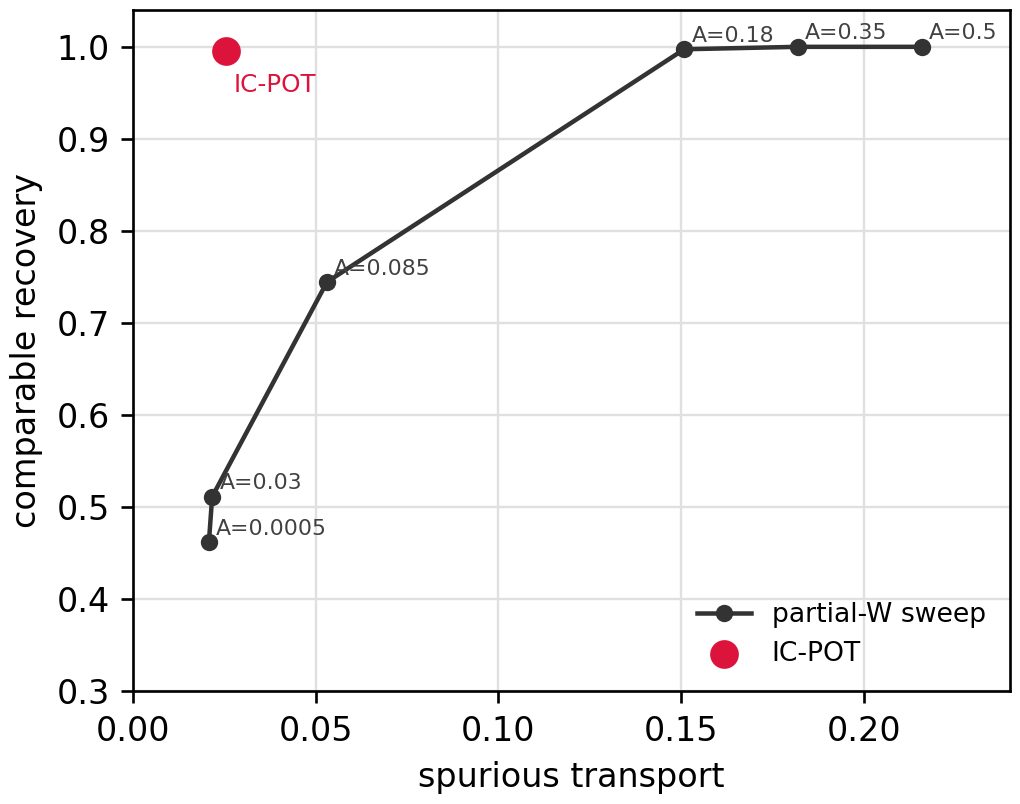}
\caption{Constant-cost partial-W trade-off on the synthetic SWIM/SAR case from Figure~\ref{fig:geophys-real-controlled}(b). The x-axis measures non-comparable mass that is still transported, and the y-axis measures recovery of comparable SAR energy. Varying the scalar partial-W price moves along a one-dimensional curve: low prices leave comparable cutoff-displaced energy aside, while high prices also transport speckle-supported or one-sided components. IC-POT uses side-specific pointwise unmatched costs under the same transport geometry and lies away from this scalar trade-off.}
\label{fig:geophys-appendix-tradeoff}
\end{figure}

\paragraph{Additional synthetic cases.}
Figure~\ref{fig:geophys-appendix-cases-a} shows representative synthetic cases from the benchmark. The number of wave systems, speckle sectors, cutoff orientations, and one-sided components vary across rows. The qualitative pattern is stable: a low global reject price leaves comparable energy aside, a high global reject price transports more non-comparable energy, and IC-POT recovers the comparable support while assigning unsupported components to the unmatched channel.

\begin{figure}[p]
\centering
\includegraphics[width=\linewidth]{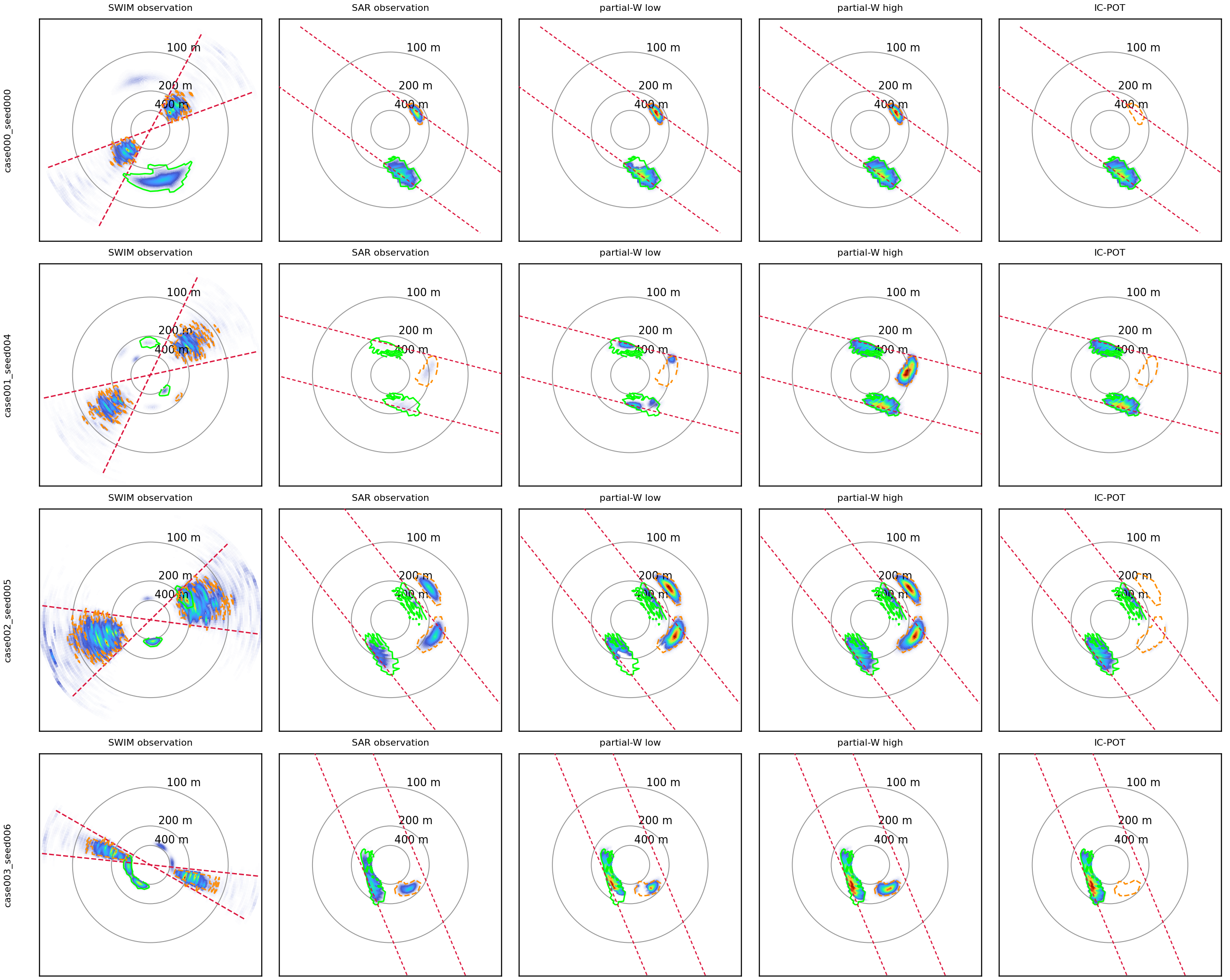}
\caption{Representative synthetic SWIM/SAR cases. Each row follows the same layout as Figure~\ref{fig:geophys-real-controlled}(b): observed SWIM, observed SAR, matched SAR mass under low partial-W, high partial-W, and IC-POT.}
\label{fig:geophys-appendix-cases-a}
\end{figure}

\clearpage
\section{Selective color transfer as a visual illustration}
\label{app:selective-color-transfer}

Selective color transfer provides a supplementary visual illustration of the same modeling language. The example is useful because the intended participation pattern is easy to inspect: one may want to choose which source regions can change, which regions should remain preserved, and which target colors may act as donors. In a constant-cost partial OT model, this selectivity must be imposed by external masks or post-processing. In IC-POT, it is encoded directly through the unmatched policies while keeping the transport geometry fixed.

Figure~\ref{fig:selective-intent}(a) shows a bilateral multi-intent edit. The source and target images, color cost, and discretization are fixed; only $(c_s,c_t)$ changes the participation pattern. Source-side activation and donor-side selection are therefore specified inside the objective. Figure~\ref{fig:selective-intent}(b) shows a second use case in which a soft segmentation signal is used continuously in $c_s$, whereas masked OT thresholds the same signal into a binary support. These examples make the mechanics of pointwise unmatched policies visually explicit.

\begin{figure}[h]
\centering
\begin{minipage}[t]{\linewidth}
\centering
\includegraphics[width=0.92\linewidth]{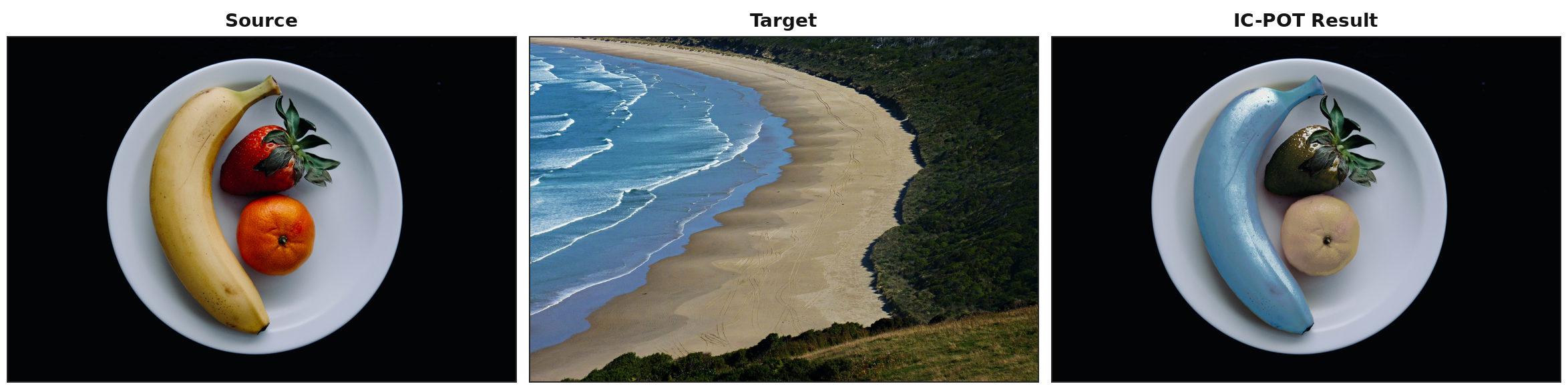}

\footnotesize (a) bilateral multi-intent transfer
\end{minipage}

\vspace{0.5em}

\begin{minipage}[t]{\linewidth}
\centering
\includegraphics[width=\linewidth]{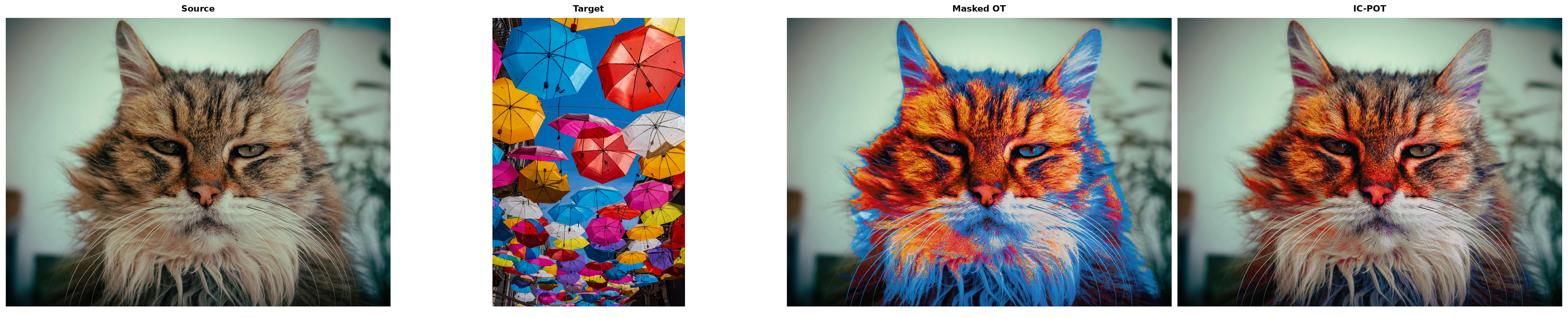}

\footnotesize (b) continuous participation vs.\ masked OT
\end{minipage}
\caption[Selective transfer under fixed transport geometry]{Supplementary selective transfer examples under fixed transport geometry. (a) Bilateral unmatched policies specify both source-side activation and donor-side selection. Appendix~\ref{app:selective-pairing-ablation} shows that changing only these policies induces different edits on the same pair. (b) On a fixed cat segmentation backend, masked OT thresholds a matte into a binary support, whereas IC-POT uses the same signal continuously in the source-side unmatched policy. Boundary metrics are reported in Appendix~\ref{app:selective-transfer-boundary}.}
\label{fig:selective-intent}
\end{figure}

\subsection{Pairing ablation}
\label{app:selective-pairing-ablation}

Figure~\ref{fig:selective-pairing-ablation} keeps fixed the same source image, the same target image, the same color cost matrix $C$, and the same transport discretization as Figure~\ref{fig:selective-intent}(a), and varies only the unmatched policies that determine which source and target subdomains are active. The figure visualizes the modeling mechanism: once participation is encoded through $(c_s,c_t)$, qualitatively different edits can be induced without changing the transport geometry itself. Panels (b) and (c) keep the active source region fixed and change only the donor region, while panel (d) keeps the donor region fixed and expands source-side participation. The ablation is therefore a compact visual check that the source-side and target-side unmatched policies control different modeling decisions.

\begin{figure}[h]
\centering
\includegraphics[width=\linewidth]{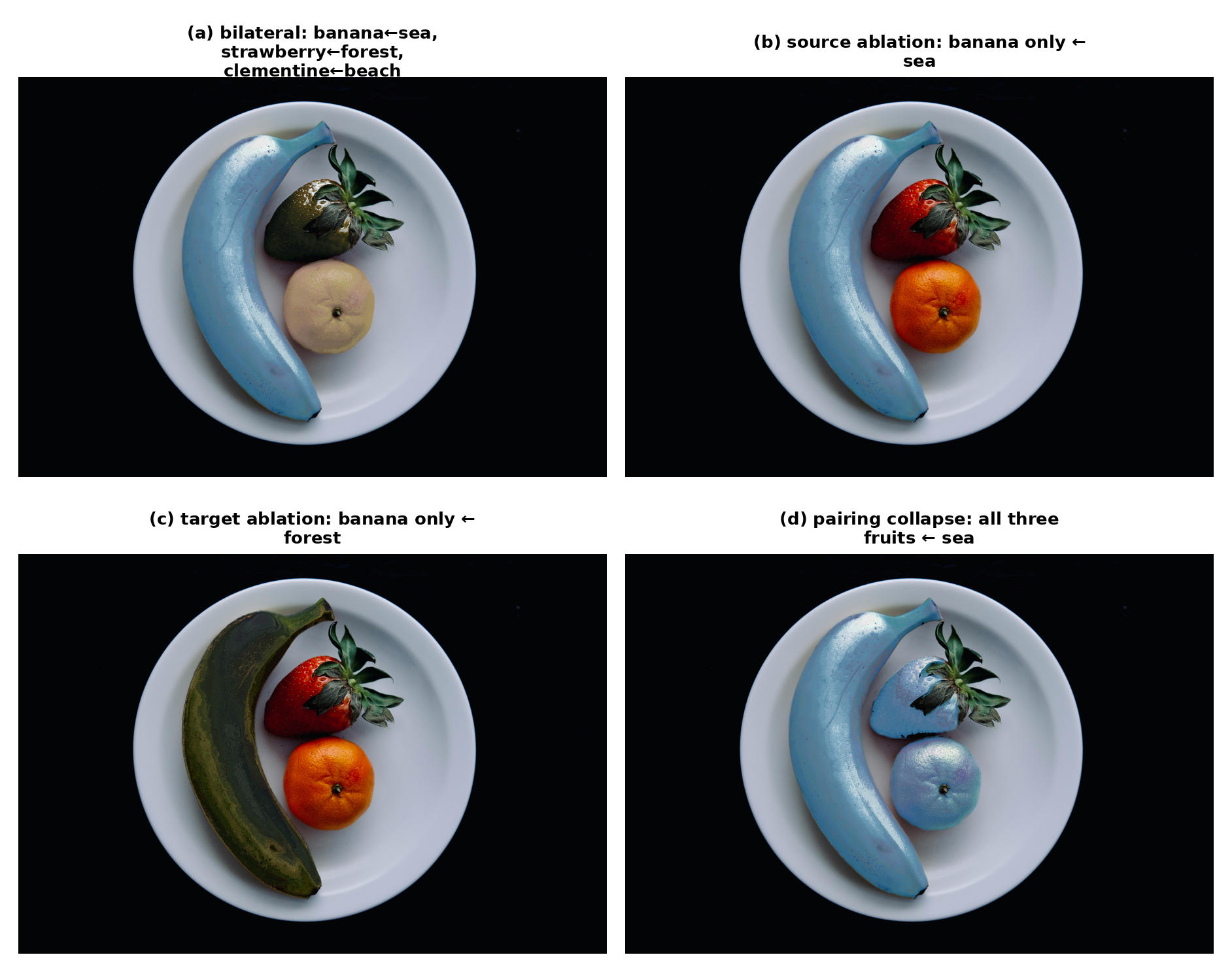}
\caption{Pairing ablation for the bilateral fruit example. The source-target pair, color cost matrix, and discretization are kept fixed; only the source-side and target-side unmatched policies are changed. Distinct edits are therefore induced by changing $(c_s,c_t)$ under the same transport geometry.}
\label{fig:selective-pairing-ablation}
\end{figure}

\subsection{Masked versus continuous participation}
\label{app:selective-transfer-masks}

Figure~\ref{fig:selective-transfer-masks} records the two source-side participation signals used in Figure~\ref{fig:selective-intent}(b). Both are derived from the same MODNet segmentation backend on the cat image. The masked-OT baseline uses the thresholded binary support, while IC-POT uses the corresponding matte continuously. This keeps the comparison focused on the transport model: the segmentation signal is held fixed, and only its use changes from hard admissibility to graded participation.

\begin{figure}[h]
\centering
\begin{minipage}[t]{0.34\linewidth}
\centering
\includegraphics[width=\linewidth]{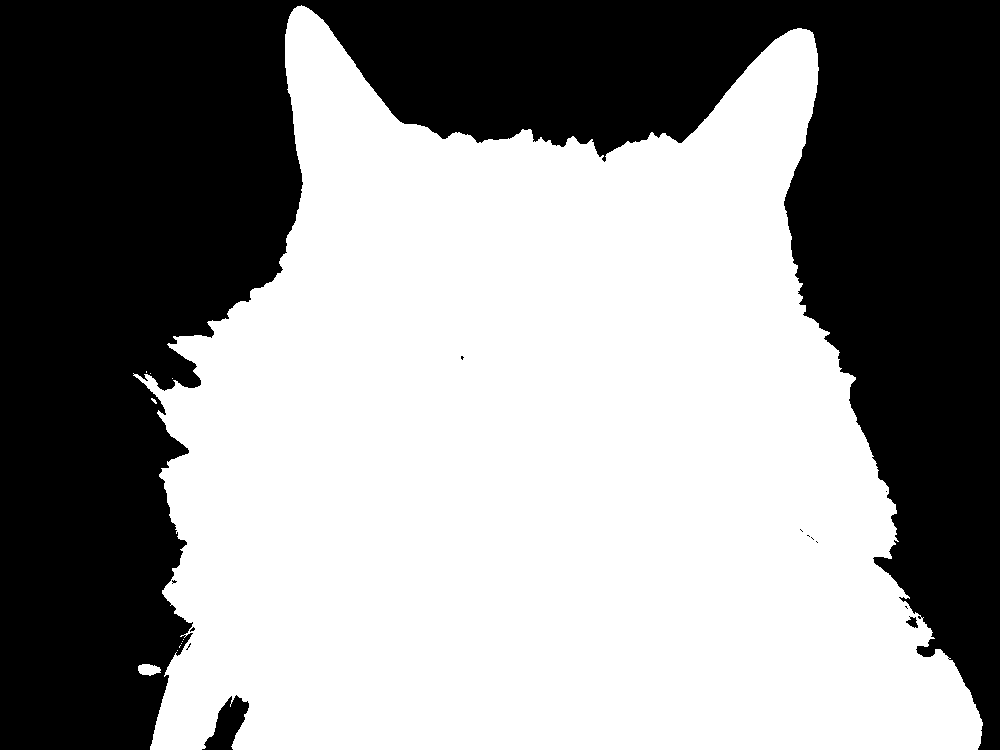}

\footnotesize Binary mask used by masked OT
\end{minipage}
\hfill
\begin{minipage}[t]{0.34\linewidth}
\centering
\includegraphics[width=\linewidth]{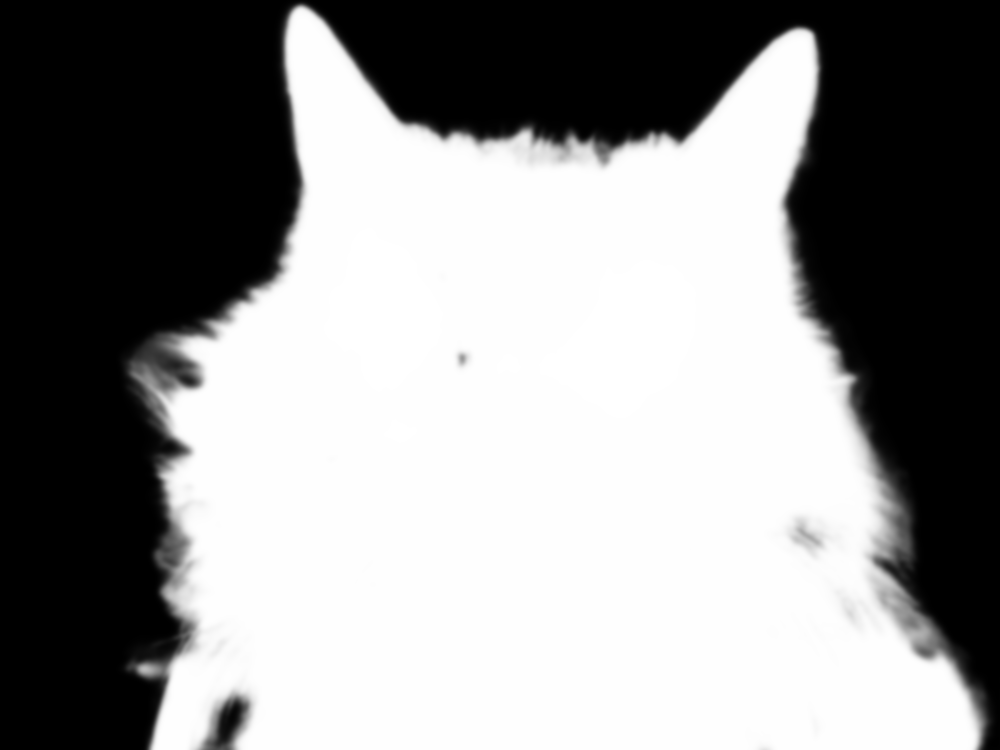}

\footnotesize Continuous matte used by IC-POT
\end{minipage}
\caption{Source-side participation signals used in the cat selective-transfer comparison. The two masks come from the same segmentation backend; masked OT thresholds the signal into a binary support, whereas IC-POT uses it continuously in the unmatched policy.}
\label{fig:selective-transfer-masks}
\end{figure}

\subsection{Boundary metrics}
\label{app:selective-transfer-boundary}

Table~\ref{tab:cat-boundary-metrics} reports a small quantitative comparison on the same cat example as Figure~\ref{fig:selective-intent}(b). The segmentation backend is held fixed, and the evaluation band is defined from the same soft matte in both cases by retaining pixels with intermediate participation values. We measure three quantities on the recoloring residual, i.e., the norm of the color change relative to the source image. The \emph{boundary residual gradient} is the average gradient magnitude of this residual inside the ambiguous band, the \emph{boundary residual local std} is the mean local standard deviation of the same residual inside the band, and \emph{outside leakage} averages the residual in a thin ring immediately outside the hard support. Lower values indicate smoother and more localized edits. These metrics are secondary to the main experiments, but they make the masked-versus-continuous distinction concrete on a fixed visual example.

\begin{table}[h]
\centering
\caption{Boundary metrics on the cat masked-vs-soft comparison. Lower is better for all three quantities.}
\label{tab:cat-boundary-metrics}
\small
\begin{tabular}{lccc}
\toprule
Method & Boundary residual gradient $\downarrow$ & Boundary residual local std $\downarrow$ & Outside leakage $\downarrow$ \\
\midrule
Masked OT & $3.58\times 10^{-2}$ & $5.67\times 10^{-2}$ & $9.61\times 10^{-3}$ \\
IC-POT & \textbf{$8.06\times 10^{-3}$} & \textbf{$1.17\times 10^{-2}$} & \textbf{$5.76\times 10^{-3}$} \\
\bottomrule
\end{tabular}
\end{table}

\section{Proofs for the Lagrangian and augmented-support formulations}
\label{app:formulation-proofs}

\subsection{Generalized Lagrangian form}

\begin{proof}[Proof of Proposition~\ref{prop:subcoupling-form}]
For any feasible triplet $(P,u,v)$ in Eq.~\eqref{eq:apot-slack}, the equality constraints imply
\[
u=\mu-P\mathbf 1_m,
\qquad
v=\nu-P^\top\mathbf 1_n.
\]
Since $u,v\ge 0$, the corresponding transport plan satisfies
\[
P\mathbf 1_m\le \mu,
\qquad
P^\top\mathbf 1_n\le \nu,
\]
that is, $P\in\Gamma_{\le}(\mu,\nu)$. Conversely, any $P\in\Gamma_{\le}(\mu,\nu)$ defines feasible slack variables by the same identities. Thus the slack formulation and the sub-coupling formulation have the same feasible degrees of freedom.

Substituting the slack variables into the objective gives
\[
\begin{aligned}
\langle C,P\rangle
&+ \langle c_s,\mu-P\mathbf 1_m\rangle
 + \langle c_t,\nu-P^\top\mathbf 1_n\rangle \\
&=
\langle c_s,\mu\rangle+\langle c_t,\nu\rangle
+ \sum_{i=1}^n\sum_{j=1}^m
\bigl(C_{ij}-c_s(i)-c_t(j)\bigr)P_{ij}.
\end{aligned}
\]
The first two terms are constant with respect to $P$, which proves the claimed equivalence.
\end{proof}

\subsection{Constant-rebate specialization}

\begin{proof}[Proof of Proposition~\ref{prop:constant-rebate-specialization}]
Assume $c_s(i)=\alpha$ for all $i$ and $c_t(j)=\beta$ for all $j$. By Proposition~\ref{prop:subcoupling-form}, IC-POT is equivalent, up to an additive constant, to
\[
\min_{P\in\Gamma_{\le}(\mu,\nu)}
\sum_{i,j}\bigl(C_{ij}-\alpha-\beta\bigr)P_{ij}
=
\min_{P\in\Gamma_{\le}(\mu,\nu)}
\langle C,P\rangle-\lambda \sum_{i,j}P_{ij},
\]
with $\lambda=\alpha+\beta$. This is exactly the constant-rebate Lagrangian form of partial transport: every transported unit receives the same reward, independently of its source and target locations.

The connection with partial-W follows by combining this constant-rebate structure with the augmented marginals used by \citet{chapel2020partial} to encode the desired transported-mass budget.
\end{proof}

\subsection{Dummy-point / slack equivalence}
\label{app:equivalence}

We prove here the augmented-support equivalence of Proposition~\ref{prop:augmented-support}. Consider again the slack formulation
\[
\min_{P,u,v} \langle C,P\rangle + \langle c_s,u\rangle + \langle c_t,v\rangle
\]
subject to
\[
P\mathbf 1_m + u = \mu,\qquad P^\top \mathbf 1_n + v = \nu,\qquad P,u,v \ge 0.
\]
We introduce the augmented marginals
\[
\bar\mu =
\begin{bmatrix}
\mu\\
\|\nu\|_1
\end{bmatrix},
\qquad
\bar\nu =
\begin{bmatrix}
\nu\\
\|\mu\|_1
\end{bmatrix},
\]
and the augmented cost matrix
\[
\bar C =
\begin{pmatrix}
C & c_s\\
c_t^\top & 0
\end{pmatrix}.
\]
The claim is that the slack formulation is equivalent to the balanced Kantorovich transport problem
\[
\min_{\bar P \in \Pi(\bar\mu,\bar\nu)} \langle \bar C,\bar P\rangle.
\]
Here
\[
\Pi(\bar\mu,\bar\nu)
:=
\left\{
\bar P \in \mathbb{R}_+^{(n+1)\times(m+1)} :
\bar P \mathbf{1}_{m+1} = \bar\mu,\;
\bar P^\top \mathbf{1}_{n+1} = \bar\nu
\right\}
\]
denotes the usual balanced Kantorovich polytope on the augmented supports.

\begin{proof}
We prove the equivalence by explicit construction in both directions.

\paragraph{From slack variables to an augmented coupling.}
Let $(P,u,v)$ be feasible for the slack formulation. Define
\[
e := \sum_{i=1}^n \sum_{j=1}^m P_{ij},
\qquad
\bar P =
\begin{pmatrix}
P & u\\
v^\top & e
\end{pmatrix}.
\]
Since $P,u,v \ge 0$, we immediately have $\bar P \ge 0$. It remains to verify the augmented marginal constraints.

For each original source index $i \in \{1,\dots,n\}$,
\[
\sum_{j=1}^{m+1} \bar P_{ij}
=
\sum_{j=1}^{m} P_{ij} + u_i
=
\mu_i
\]
by the slack constraint. For the dummy source row, we have
\[
\sum_{j=1}^{m+1} \bar P_{n+1,j}
=
\sum_{j=1}^{m} v_j + e.
\]
Using
\[
\sum_{j=1}^{m} v_j
=
\|\nu\|_1 - \sum_{i=1}^n \sum_{j=1}^m P_{ij}
=
\|\nu\|_1 - e,
\]
we obtain
\[
\sum_{j=1}^{m+1} \bar P_{n+1,j} = \|\nu\|_1.
\]
Hence $\bar P \mathbf{1}_{m+1} = \bar\mu$.

Similarly, for each original target index $j \in \{1,\dots,m\}$,
\[
\sum_{i=1}^{n+1} \bar P_{ij}
=
\sum_{i=1}^{n} P_{ij} + v_j
=
\nu_j.
\]
For the dummy target column,
\[
\sum_{i=1}^{n+1} \bar P_{i,m+1}
=
\sum_{i=1}^{n} u_i + e.
\]
Using
\[
\sum_{i=1}^{n} u_i
=
\|\mu\|_1 - \sum_{i=1}^n \sum_{j=1}^m P_{ij}
=
\|\mu\|_1 - e,
\]
we obtain
\[
\sum_{i=1}^{n+1} \bar P_{i,m+1} = \|\mu\|_1.
\]
Thus $\bar P^\top \mathbf{1}_{n+1} = \bar\nu$, so $\bar P \in \Pi(\bar\mu,\bar\nu)$.

Finally, the objective is preserved exactly by block expansion:
\[
\langle \bar C,\bar P\rangle
=
\langle C,P\rangle + \langle c_s,u\rangle + \langle c_t,v\rangle + 0\cdot e
=
\langle C,P\rangle + \langle c_s,u\rangle + \langle c_t,v\rangle.
\]

\paragraph{From an augmented coupling to slack variables.}
Conversely, let
\[
\bar P =
\begin{pmatrix}
P & u\\
v^\top & e
\end{pmatrix}
\in \Pi(\bar\mu,\bar\nu).
\]
Since $\bar P \ge 0$, we directly obtain $P \ge 0$, $u \ge 0$, and $v \ge 0$.

For each source index $i \in \{1,\dots,n\}$,
\[
\sum_{j=1}^{m} P_{ij} + u_i
=
\sum_{j=1}^{m+1} \bar P_{ij}
=
\bar\mu_i
=
\mu_i,
\]
hence
\[
P\mathbf 1 + u = \mu.
\]
Likewise, for each target index $j \in \{1,\dots,m\}$,
\[
\sum_{i=1}^{n} P_{ij} + v_j
=
\sum_{i=1}^{n+1} \bar P_{ij}
=
\bar\nu_j
=
\nu_j,
\]
hence
\[
P^\top \mathbf 1 + v = \nu.
\]
Therefore $(P,u,v)$ is feasible for the slack formulation.

Once again, the objective is preserved by construction:
\[
\langle \bar C,\bar P\rangle
=
\langle C,P\rangle + \langle c_s,u\rangle + \langle c_t,v\rangle.
\]
This proves the equivalence of the two optimization problems.
\end{proof}

This appendix only establishes equivalence with a balanced Kantorovich transport problem on an augmented support. It does not claim any metric or Wasserstein-distance property for general pointwise unmatched costs.

\clearpage
\section{Proofs for the Dual and Structural Results of Section~3}
\label{app:dual-structure-proofs}

\begin{proof}[Proof of Proposition~\ref{prop:dual}]
Introduce dual variables $f \in \mathbb{R}^n$ and $g \in \mathbb{R}^m$ for the equality constraints
\[
P \mathbf 1 + u = \mu,
\qquad
P^\top \mathbf 1 + v = \nu.
\]
The Lagrangian is
\[
\mathcal{L}(P,u,v;f,g)
=
\sum_{i,j} C_{ij} P_{ij}
+ \sum_i c_s(i) u_i
+ \sum_j c_t(j) v_j
+ \sum_i f_i\Bigl(\mu_i - \sum_j P_{ij} - u_i\Bigr)
+ \sum_j g_j\Bigl(\nu_j - \sum_i P_{ij} - v_j\Bigr).
\]
Rearranging terms gives
\[
\mathcal{L}(P,u,v;f,g)
=
\sum_i f_i \mu_i + \sum_j g_j \nu_j
+ \sum_{i,j} \bigl(C_{ij} - f_i - g_j\bigr) P_{ij}
+ \sum_i \bigl(c_s(i)-f_i\bigr) u_i
+ \sum_j \bigl(c_t(j)-g_j\bigr) v_j.
\]
Since $P,u,v \ge 0$, the infimum of the Lagrangian over $(P,u,v)$ is finite if and only if
\[
C_{ij} - f_i - g_j \ge 0,\qquad
c_s(i)-f_i \ge 0,\qquad
c_t(j)-g_j \ge 0.
\]
Under these conditions, the infimum is attained at $P=u=v=0$ and equals
\[
\sum_i f_i \mu_i + \sum_j g_j \nu_j.
\]
This yields the dual problem in Eq.~\eqref{eq:apot-dual}.
\end{proof}

\begin{proof}[Proof of Proposition~\ref{prop:cs}]
Problem~\eqref{eq:apot-slack} and its dual \eqref{eq:apot-dual} are a primal--dual pair of linear programs, so optimal primal and dual solutions satisfy complementary slackness. Using the coefficient form from the Lagrangian proof above, we obtain
\[
\bigl(C_{ij}-f_i-g_j\bigr) P_{ij}=0,
\qquad
\bigl(c_s(i)-f_i\bigr) u_i=0,
\qquad
\bigl(c_t(j)-g_j\bigr) v_j=0.
\]
Therefore $P_{ij}>0$ implies $f_i+g_j=C_{ij}$, $u_i>0$ implies $f_i=c_s(i)$, and $v_j>0$ implies $g_j=c_t(j)$.
\end{proof}

\begin{proof}[Proof of Proposition~\ref{prop:dominated}]
Fix $(i,j)$ and let $(P,u,v)$ be any feasible solution. Set $\delta := P_{ij}$ and define a new feasible point by
\[
P'_{ij}=0,\qquad
P'_{k\ell}=P_{k\ell}\ \text{for } (k,\ell)\neq(i,j),
\qquad
u'_i=u_i+\delta,
\qquad
v'_j=v_j+\delta,
\]
with all other coordinates of $u'$ and $v'$ unchanged. Feasibility is preserved because the row and column balances at $i$ and $j$ are maintained exactly.

The objective variation is
\[
\Delta
=
\Bigl(\sum_{k,\ell} C_{k\ell} P'_{k\ell} + \sum_k c_s(k) u'_k + \sum_\ell c_t(\ell) v'_\ell\Bigr)
- \Bigl(\sum_{k,\ell} C_{k\ell} P_{k\ell} + \sum_k c_s(k) u_k + \sum_\ell c_t(\ell) v_\ell\Bigr)
=
\delta\bigl(c_s(i)+c_t(j)-C_{ij}\bigr).
\]
If $C_{ij}>c_s(i)+c_t(j)$ and $\delta>0$, then $\Delta<0$, contradicting optimality. Hence every optimal solution must satisfy $P_{ij}=0$.

If $C_{ij}=c_s(i)+c_t(j)$, then $\Delta=0$, so the same local modification removes the edge without changing the objective value. Starting from any optimal solution, one obtains another optimal solution with $P_{ij}=0$.
\end{proof}

\begin{proof}[Strict separation from constant-cost partial OT]
Consider the instance with two source points and one target point:
\[
\mu=(1,1),\qquad \nu=(1),\qquad
C=
\begin{bmatrix}
0.3\\
0.3
\end{bmatrix},
\]
and unmatched costs
\[
c_s(1)=1,\qquad c_s(2)=0,\qquad c_t(1)=1.
\]
Under IC-POT, matching the first source point and leaving the second unmatched yields objective
\[
0.3 + 0 = 0.3.
\]
Matching the second source point and leaving the first unmatched yields objective
\[
0.3 + 1 = 1.3.
\]
Leaving everything unmatched yields objective $1+0+1=2$. Hence the unique optimal IC-POT plan matches the first source point to the target and leaves the second unmatched.

Now consider any constant-cost model with uniform unmatched penalties $\alpha$ on the source and $\beta$ on the target, on the same transport matrix $C$. The two source points are then perfectly symmetric: they have the same mass and the same transport cost to the unique target point. Therefore, if a plan matching the first source point and leaving the second unmatched is optimal, then the swapped plan matching the second and leaving the first unmatched has exactly the same objective value. No constant-cost model can therefore reproduce the unique IC-POT optimum above.
\end{proof}

\clearpage
\section{Optimization details and sparse solver}
\label{app:optimization}

The exact reference solver used in the paper is the linear program in Eq.~\eqref{eq:apot-slack}. In implementation terms, this simply means optimizing over the transport variables $P_{ij}$ together with the slack variables $u_i$ and $v_j$, with no entropic smoothing and no approximation of the objective.

The sparse variant keeps the same objective and the same slack variables, but restricts transport to an admissible edge set
\[
E := \{(i,j) : C_{ij} \le c_s(i) + c_t(j)\}.
\]
The sparse LP is therefore
\[
\min_{P,u,v}\;
\sum_{(i,j)\in E} C_{ij} P_{ij}
 + \sum_{i=1}^n c_s(i) u_i
 + \sum_{j=1}^m c_t(j) v_j
\]
subject to
\[
\sum_{j:(i,j)\in E} P_{ij} + u_i = \mu_i,
\qquad
\sum_{i:(i,j)\in E} P_{ij} + v_j = \nu_j,
\qquad
P,u,v \ge 0.
\]
Equivalently, one may view the sparse solver as fixing $P_{ij}=0$ on all edges outside $E$ and solving the same LP on the reduced support.

\begin{proposition}
\label{prop:sparse-eq}
Let $E=\{(i,j): C_{ij} \le c_s(i)+c_t(j)\}$. Then the sparse LP on $E$ has the same optimal value as the full LP in Eq.~\eqref{eq:apot-slack}. Moreover, every sparse optimum is optimal for the full problem, and there exists a full optimum that is feasible for the sparse problem.
\end{proposition}

\begin{proof}
The sparse feasible set is a subset of the full feasible set, since it adds the constraints $P_{ij}=0$ for $(i,j)\notin E$. Therefore every sparse optimum is feasible for the full problem, and the full optimum value is at most the sparse optimum value.

Conversely, if $(i,j)\notin E$, then $C_{ij}>c_s(i)+c_t(j)$. By Proposition~\ref{prop:dominated}, every optimal solution of the full problem satisfies $P_{ij}=0$ on such edges. Hence every full optimum is already feasible for the sparse problem. It follows that the sparse optimum value is at most the full optimum value.

The two optimal values are therefore equal. The first claim follows because a sparse optimum is feasible for the full problem with the same objective value, and the second follows because every full optimum is feasible for the sparse problem.
\end{proof}

In the implementation used in this paper, the sparse solver is therefore not a different algorithmic objective. It is the same LP solved on a reduced admissible support, with optimality preserved by Proposition~\ref{prop:sparse-eq}. Numerical details such as tie-break perturbations only affect degeneracy resolution and not the model itself.

The practical size of the admissible set
\[
E=\{(i,j): C_{ij}\le c_s(i)+c_t(j)\}
\]
still depends on the geometry of the transport cost. For instance, if $c_s(i)\le K_s$ and $c_t(j)\le K_t$ uniformly, then
\[
|E|
\le
\#\{(i,j): C_{ij}\le K_s+K_t\},
\]
so sparsity is controlled by how many source-target pairs fall below this global threshold. In geometrically localized regimes, this set can be much smaller than the full Cartesian product, whereas in diffuse regimes it may remain large even though the sparse LP stays exact.

\clearpage
\section{Instantiation principles for unmatched functions}
\label{app:intent-instantiation}

This appendix summarizes the instantiation strategy used throughout the paper. It gives a reproducible design framework: unmatched functions are built from simple, semantically grounded signals in the original domain, and used either to specify participation or to price rejection. In both cases, the goal is to replace arbitrary design in the augmented support by low-dimensional profiles whose meaning is clear before optimization.

\paragraph{Active-support shaping.}
In the first regime, the unmatched functions define which parts of the two marginals are intended to stay available for transport. The relevant design object is a support-level participation profile. A practical recipe is to start from descriptors that identify the intended active subdomains on each side, normalize them into a common range, then map them monotonically to low reject cost where participation should remain easy and high reject cost where participation should be discouraged. These descriptors may come from physical reliability maps, masks, mattes, semantic region indicators, donor-side profiles, or other task-side signals that already live in the original domain. A constant-cost baseline remains an important reference and is structurally limited in this regime because a single uniform unmatched rule cannot encode asymmetric source and target participation patterns. This is the regime instantiated by the geophysical experiment and illustrated visually by the selective-transfer appendix.

\paragraph{Rejection-pressure modulation.}
In the second regime, the ambient support is fixed and the unmatched function instead modulates how expensive it is to reject each point. The design problem is then to construct a task-side relevance score $r_i \in [0,1]$ and map it to unmatched costs, for example
\[
c_s(i)=c_{\min} + (c_{\max}-c_{\min})\,r_i.
\]
This reduces the problem to the construction of a scalar profile. In practice, larger values of $r_i$ should correlate with stronger evidence that point $i$ belongs to the part of the support that should remain matchable, while smaller values should make rejection easier. Proxies for confidence, local support, structural consistency, or selection bias all fall into this family. When $r_i$ is constant, this reduces to the constant-rebate model of Proposition~\ref{prop:constant-rebate-specialization}. The PU and OPDA experiments instantiate this regime with three concrete examples: a selection-bias proxy in PU, posterior confidence in entropy-based OPDA, and neighborhood consistency in prototype-support OPDA. In all current experiments, $c_t$ is kept uniform in this regime because the non-uniform design sits on the side where rejection is the meaningful modeling variable.

These constructions make prior design structured and explicit. The paper provides a design language for turning task-side signals into unmatched profiles, clarifies the role those profiles are meant to play, and leaves fully automatic instantiation across tasks as an open direction.

\clearpage
\section{PU Learning}
\label{app:pu-selection-bias-extra}

\subsection{Setup and unmatched policy}

This appendix complements the controlled PU selection-bias experiment from the main text.

Let $X \in \mathbb{R}^2$ denote the features, $Y \in \{+1,-1\}$ the latent class label, and $S \in \{0,1\}$ the indicator that a positive point is actually observed. The observed positive set is
\[
\{X : Y=+1,\ S=1\},
\]
whereas the unlabeled pool is
\[
\{X : Y=+1,\ S=0\} \cup \{X : Y=-1\}.
\]
The positive-selection mechanism is
\[
\rho(x)=\mathbb{P}(S=1 \mid X=x,\ Y=+1).
\]
Under SCAR, $\rho(x)$ is constant; under selected-at-random sampling, it depends on the covariates. In the synthetic construction used in the paper, the latent supports are held fixed across regimes: positives occupy a horizontal band, while negatives form two lateral modes pulled toward the center. The homogeneous regime is obtained with nearly uniform $\rho(x)$, whereas the heterogeneous regime concentrates $\rho(x)$ near the center and lowers it in the fringes.

\paragraph{Unmatched policies.}
Unless noted otherwise, all results use the same light 5-seed setup as Figure~\ref{fig:pu-toy}, the same transported-mass prior $s=\pi$, and the same constant baseline (\emph{partial-W}) with
\[
c_s(x)=c_t(y)=A,\qquad A=0.15.
\]
The structured policy is tied to the selected-at-random mechanism. In the heterogeneous regime, the observation process selects central positives more often than fringe positives. A low density of observed positives near the fringes is therefore a sampling artifact rather than evidence that fringe points are negative. The source-side unmatched cost is chosen to encode exactly this asymmetry: rejection becomes more expensive in the under-observed fringe regions, while the target-side cost remains uniform because the modeled bias concerns the source-side positive support.
The aligned IC-POT policy keeps $c_t$ fixed and modulates only the source-side unmatched cost through
\[
c_s(x)=c_{\min}+(c_{\max}-c_{\min})\,\widetilde{|x_1|},
\qquad c_{\min}=0.01,\quad c_{\max}=0.10,
\]
where $\widetilde{|x_1|}$ denotes the normalized horizontal magnitude, i.e., the rescaling of $|x_1|$ to $[0,1]$ over the source support. This makes discarding lateral positive mass progressively more costly toward the fringes, which matches the intended selected-at-random mechanism. The deliberately misaligned control used below simply reverses this profile so that rejection becomes cheaper in the fringes and costlier near the center.

\subsection{Additional experiments}

\paragraph{Continuous selection-bias transition.}
We interpolate continuously from SCAR-like sampling to covariate-dependent selection by fixing the central positive-selection probability $\rho_{\mathrm{center}}=\rho(x_1=0)=0.95$ and decreasing the fringe probability $\rho_{\mathrm{fringe}}$, i.e., the selection probability in the far lateral regions of the positive band, from $0.95$ to $0.03$. Figure~\ref{fig:pu-sar-transition} reports the mean F1 gap between IC-POT and \emph{partial-W} on the heterogeneous regime. The gain of the aligned IC-POT policy grows steadily as the selection bias strengthens: it is small under SCAR-like sampling and becomes large once the fringes are strongly under-observed. For comparison, we also plot a deliberately misaligned control that favors the center instead of the fringes; its gap remains negative or much smaller throughout most of the sweep.

\begin{figure}[h]
\centering
\includegraphics[width=0.58\linewidth]{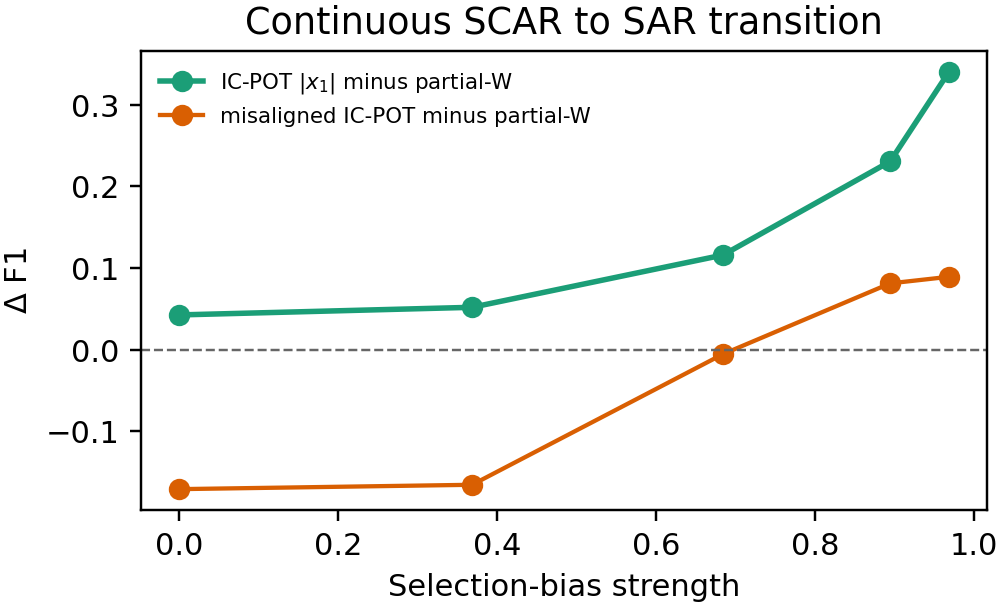}
\caption{Continuous selection-bias sweep. The x-axis measures the strength of the selection bias through $1-\rho_{\mathrm{fringe}}/\rho_{\mathrm{center}}$. We report the mean F1 gap with respect to \emph{partial-W} on the heterogeneous regime.}
\label{fig:pu-sar-transition}
\end{figure}

\paragraph{Sensitivity to negative geometry.}
We next vary the vertical offset of the negative side modes while keeping the positive support and the selection mechanism fixed. Figure~\ref{fig:pu-sar-neg-geometry} shows that the advantage of the aligned IC-POT policy is strongest when the negatives are close enough to compete with under-observed fringe positives, and decreases as they move farther away from the positive band. This is exactly the intended mechanism: once the negatives become less competitive, the constant global rule becomes less problematic.

\begin{figure}[h]
\centering
\includegraphics[width=0.82\linewidth]{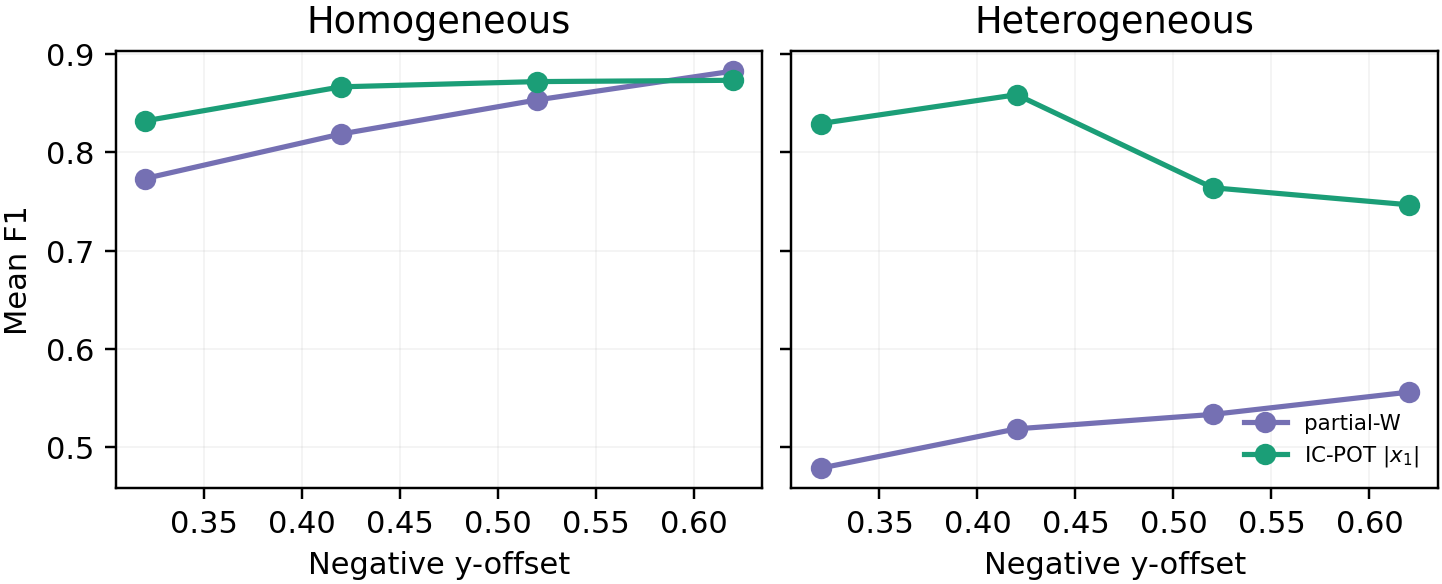}
\caption{Sensitivity to the negative geometry. We vary only the vertical offset of the negative modes and report mean F1 over 5 seeds in the homogeneous and heterogeneous regimes.}
\label{fig:pu-sar-neg-geometry}
\end{figure}

\paragraph{Hyperparameter sensitivity.}
Table~\ref{tab:pu-sar-hparams} reports a small local sweep around the retained hyperparameters. The constant baseline remains confined to a narrow plateau, whereas the IC-POT policy is clearly best near the retained bounded profile $(c_{\min},c_{\max})=(0.01,0.10)$. This is also the most stable setting we found that preserves the homogeneous case while remaining strongly advantageous in the heterogeneous one.

\begin{table}[h]
\centering
\small
\begin{tabular}{lccc}
\toprule
Method / params & Homogeneous F1 & Heterogeneous F1 & Mean \\
\midrule
partial-W, $A=0.13$ & 0.81 & 0.52 & 0.67 \\
partial-W, $A=0.14$ & 0.81 & 0.53 & 0.67 \\
partial-W, $A=0.15$ & 0.82 & 0.52 & 0.67 \\
partial-W, $A=0.16$ & 0.82 & 0.51 & 0.66 \\
\midrule
IC-POT, $c_{\min}=0.005, c_{\max}=0.10$ & 0.87 & 0.86 & 0.86 \\
IC-POT, $c_{\min}=0.005, c_{\max}=0.11$ & 0.86 & 0.86 & 0.86 \\
IC-POT, $c_{\min}=0.010, c_{\max}=0.10$ & 0.87 & 0.86 & 0.86 \\
IC-POT, $c_{\min}=0.010, c_{\max}=0.11$ & 0.86 & 0.86 & 0.86 \\
\bottomrule
\end{tabular}
\caption{Small hyperparameter sweep on the 5-seed PU-SAR setup.}
\label{tab:pu-sar-hparams}
\end{table}

\paragraph{Misaligned unmatched policy.}
Table~\ref{tab:pu-sar-control} shows a negative control in which the unmatched profile is intentionally misaligned with the selection mechanism by favoring the center instead of the fringes. The gain disappears, which supports the main claim of the section: performance depends on alignment with the covariate structure of the selection bias.

\begin{table}[h]
\centering
\small
\begin{tabular}{lcc}
\toprule
Method & Homogeneous F1 & Heterogeneous F1 \\
\midrule
partial-W & 0.82 & 0.52 \\
IC-POT $|x_1|$ & 0.87 & 0.86 \\
misaligned IC-POT & 0.50 & 0.61 \\
\bottomrule
\end{tabular}
\caption{Negative control on the 5-seed PU-SAR setup: replacing the fringe-favoring profile by a center-favoring unmatched policy removes the gain.}
\label{tab:pu-sar-control}
\end{table}

\paragraph{Additional qualitative example.}
Figure~\ref{fig:pu-sar-qual-control} shows an additional heterogeneous example comparing \emph{partial-W}, the aligned IC-POT policy, and the misaligned control on the same seed. The aligned policy preserves the intended fringe participation while avoiding many of the central false positives produced by the uniform baseline; the misaligned control does not.

\begin{figure}[h]
\centering
\includegraphics[width=0.98\linewidth]{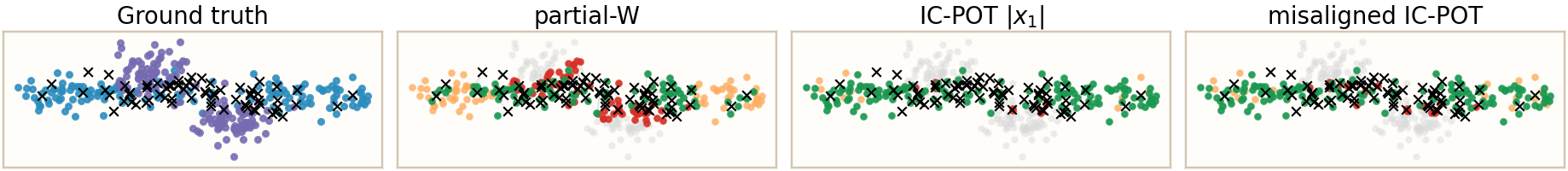}
\caption{Additional qualitative comparison on a heterogeneous PU selection-bias seed. From left to right: ground truth, \emph{partial-W}, aligned IC-POT with $|x_1|$, and a misaligned center-favoring control.}
\label{fig:pu-sar-qual-control}
\end{figure}
\clearpage
\section{Unmatched functions used in the OPDA runs}
\label{app:opda-unmatched}

This appendix records the exact unmatched policies used for the OPDA results reported in the main text. In all cases, we start from the calibrated CLIP distillation posterior $p_i \in \Delta^{K-1}$ of target sample $i$ over the $K$ source prototypes, and from the corresponding target feature vectors $z_i$. The predicted prototype label is $\hat y_i = \arg\max_c p_{ic}$. The transport cost is kept fixed and only the unmatched policy is changed.

The design principle is to protect target samples that carry evidence of belonging to the shared support. In OPDA, target-private samples should be rejected, while target samples compatible with source classes should remain available for matching. The available evidence comes from the fixed representation: posterior confidence measures how clearly a sample is assigned to a source prototype, and local neighborhood support measures whether this assignment is reinforced by nearby target samples. The two IC-POT variants instantiate these two signals separately.

\paragraph{Constant-cost baseline (\emph{partial-W}).}
The baseline uses a uniform unmatched policy on both marginals,
\[
c_s(i)=A,\qquad c_t(j)=A.
\]
In the reported OPDA runs, $A=0.5$.

\paragraph{Entropy-based IC-POT.}
Let
\[
\bar H(p_i)= -\frac{1}{\log K}\sum_{c=1}^K p_{ic}\log p_{ic}
\]
denote the normalized posterior entropy, and let
\[
q_i = 1-\bar H(p_i)
\]
be the corresponding confidence score. The entropy-based source-side unmatched cost is
\[
c_s(i)=A+\lambda_{\mathrm{ent}}(2q_i-1),
\]
with
\[
c_t(j)=A.
\]
This profile makes rejection more expensive for samples with low posterior entropy and cheaper for ambiguous samples. It is the simplest source-side policy one can build from calibrated posteriors, and it tests whether pointwise confidence already contains useful reject structure beyond a global threshold.
This is the \texttt{constant\_plus\_entropy} variant in the implementation. The retained parameters for the paper runs are
\[
A=0.5,\qquad \lambda_{\mathrm{ent}}=0.3.
\]

\paragraph{Prototype-support IC-POT.}
For each target sample $i$, let $\mathcal N_k(i)$ be its $k$ nearest neighbors in normalized feature space. The prototype-support score is defined as
\[
r_i = \frac{1}{k}\sum_{\ell \in \mathcal N_k(i)} p_{\ell,\hat y_i},
\]
that is, the average posterior mass assigned by the local neighborhood to the prototype currently predicted for sample $i$. The resulting unmatched cost is
\[
c_s(i)=A+\lambda_{\mathrm{ps}}\, r_i,
\]
again with
\[
c_t(j)=A.
\]
This profile protects samples whose predicted prototype is supported by the surrounding target geometry. The motivation is that OPDA errors are often structural: an isolated confident prediction is weaker evidence than a locally coherent target region that points to the same source prototype. Prototype-support therefore uses neighborhood consistency as a proxy for shared-class support.
This is the \texttt{constant\_plus\_prototype\_support\_pos} variant in the implementation. The retained parameters for the paper runs are
\[
A=0.5,\qquad \lambda_{\mathrm{ps}}=0.7,\qquad k=15.
\]

\paragraph{Implementation note.}
In the UniOOD wrapper used for the paper runs, these OPDA variants are instantiated by setting
\[
\texttt{ic-pot\_t\_low}=\texttt{ic-pot\_t\_mid}=\texttt{ic-pot\_t\_high}=A,
\]
so that the target-side unmatched cost remains constant even though the codebase also supports more general non-uniform target-side policies. This keeps the OPDA comparison deliberately conservative: the only non-uniform component is the source-side reject policy on target samples.

\clearpage
\section{OPDA Protocol Summary}
\label{app:opda-protocol}

Table~\ref{tab:opda-splits} summarizes the standard OPDA splits used in our experiments. They match the protocol followed by the CLIP distillation benchmark and are kept fixed across all IC-POT variants. Office-31 and Office-Home are averaged over their standard transfer sets. VisDA contains a single transfer (`Syn$\rightarrow$Real`) and is reported accordingly. For DomainNet, we follow the standard multi-seed evaluation protocol and report averages over transfers and seeds.

\begin{table}[h]
\centering
\caption{Standard OPDA splits used throughout the paper. The number of target-private classes is determined by the remaining classes once shared and source-private classes are fixed.}
\label{tab:opda-splits}
\small
\begin{tabular}{lcccc}
\toprule
Dataset & Domains / transfers & Shared & Source-private & Target-private \\
\midrule
Office-31 & 3 domains / 6 transfers & 10 & 10 & 11 \\
Office-Home & 4 domains / 12 transfers & 10 & 5 & 50 \\
VisDA & Syn $\rightarrow$ Real / 1 transfer & 6 & 3 & 3 \\
DomainNet & 3 domains / 6 transfers & 150 & 50 & 145 \\
\bottomrule
\end{tabular}
\end{table}

\clearpage
\section{Structural metrics used in the OPDA analysis}
\label{app:structure-metrics}

This appendix defines the structural metrics used to interpret the OPDA results. They are not optimization objectives. Their role is descriptive: they summarize how diffuse the calibrated posteriors are and how coherent the local target geometry remains.

Consider a target sample $x_i$ with calibrated class logits $\ell_i \in \mathbb{R}^K$, where $K$ is the number of source classes. Let
\[
p_{ik} = \frac{\exp(\ell_{ik})}{\sum_{k'=1}^K \exp(\ell_{ik'})}
\]
be the corresponding posterior over source classes.

\paragraph{Normalized entropy.}
For each sample, we compute the entropy
\[
H_i = - \sum_{k=1}^K p_{ik}\log p_{ik},
\]
and normalize it by the maximum entropy $\log K$:
\[
\widetilde H_i = \frac{H_i}{\log K}.
\]
The reported \emph{mean normalized entropy} is the dataset average of $\widetilde H_i$. Smaller values indicate sharper posteriors.

\paragraph{Top-1 / top-2 margin.}
Let $\ell_i^{(1)}$ and $\ell_i^{(2)}$ denote the largest and second-largest entries of $\ell_i$. The margin is
\[
m_i = \ell_i^{(1)} - \ell_i^{(2)}.
\]
The reported \emph{mean margin} is the dataset average of $m_i$. Larger values indicate better local separability.

\paragraph{k-NN label consistency.}
Let $z_i$ denote the representation used for the structural analysis, and let $\mathcal N_k(i)$ be the set of the $k$ nearest neighbors of $z_i$ in Euclidean distance (excluding $i$ itself). With target ground-truth labels $y_i$, the local label consistency is
\[
\mathrm{Cons}_{\mathrm{label}}
=
\frac{1}{N}\sum_{i=1}^N \frac{1}{k}\sum_{j\in \mathcal N_k(i)} \mathbf 1\{y_j = y_i\}.
\]
In our experiments we use $k=10$. Larger values indicate stronger local class coherence.

\paragraph{k-NN known/unknown consistency.}
Let $b_i \in \{0,1\}$ denote the binary indicator of whether sample $i$ belongs to a source-shared class. Using the same neighborhood graph,
\[
\mathrm{Cons}_{\mathrm{K/U}}
=
\frac{1}{N}\sum_{i=1}^N \frac{1}{k}\sum_{j\in \mathcal N_k(i)} \mathbf 1\{b_j = b_i\}.
\]
This metric summarizes the local coherence of the known-versus-unknown partition.

In the OPDA discussion, high margins and high consistency values indicate a cleaner local target structure, while high normalized entropy indicates a more diffuse regime.

\begin{table}[h]
\centering
\caption{Dataset-level structural summaries used in the OPDA discussion. Values are averaged over the transfer-level analyses reported in the joint CLIP-feature evaluation.}
\label{tab:opda-structure-metrics}
\small
\begin{tabular}{lcccc}
\toprule
Dataset & Mean norm.\ entropy & Mean margin & k-NN label cons. & k-NN K/U cons. \\
\midrule
Office-31 & 0.41 & 2.30 & 0.84 & 0.98 \\
Office-Home & 0.60 & 1.68 & 0.56 & 0.96 \\
VisDA & 0.46 & 2.22 & 0.79 & 0.86 \\
DomainNet & 0.38 & 2.49 & 0.36 & 0.68 \\
\bottomrule
\end{tabular}
\end{table}

\clearpage
\section{Per-transfer OPDA results}
\label{app:opda-per-transfer}

Tables~\ref{tab:opda-office31-per-transfer}, \ref{tab:opda-officehome-per-transfer}, and~\ref{tab:opda-domainnet-per-transfer} report the per-transfer H and H$^3$ scores for the three transport variants retained in the main text. We include these tables to make the transfer-level behavior explicit, especially for the comparison between the constant-cost baseline and its non-uniform refinements.

\begin{table}[h]
\centering
\caption{Per-transfer OPDA results on Office-31. Each entry reports H / H$^3$.}
\label{tab:opda-office31-per-transfer}
\small
\begin{tabular}{lcccc}
\toprule
Transfer & CLIP dis + entropy-only & partial-W & entropy-based & prototype-support \\
\midrule
A$\rightarrow$D & 91.59 / 91.67 & 91.80 / 91.80 & 92.30 / 92.14 & \textbf{93.78} / \textbf{93.12} \\
A$\rightarrow$W & 88.66 / 91.09 & 93.24 / 94.26 & 93.24 / 94.26 & \textbf{94.80} / \textbf{95.32} \\
D$\rightarrow$A & 90.30 / 82.67 & 96.68 / 86.14 & \textbf{96.99} / \textbf{86.30} & 96.77 / 86.19 \\
D$\rightarrow$W & 87.48 / 90.25 & 94.03 / 94.79 & 94.03 / 94.79 & \textbf{94.80} / \textbf{95.32} \\
W$\rightarrow$A & 91.31 / 83.23 & 96.73 / 86.16 & \textbf{96.78} / \textbf{86.19} & 96.77 / 86.19 \\
W$\rightarrow$D & 93.29 / 92.79 & 92.03 / 91.96 & 92.03 / 91.96 & \textbf{93.78} / \textbf{93.12} \\
\bottomrule
\end{tabular}
\end{table}

\begin{table}[h]
\centering
\caption{Per-transfer OPDA results on Office-Home. Each entry reports H / H$^3$.}
\label{tab:opda-officehome-per-transfer}
\small
\begin{tabular}{lcccc}
\toprule
Transfer & CLIP dis + entropy-only & partial-W & entropy-based & prototype-support \\
\midrule
Ar$\rightarrow$Cl & 85.51 / 80.83 & 89.93 / \textbf{83.42} & \textbf{89.94} / \textbf{83.42} & 88.28 / 82.47 \\
Ar$\rightarrow$Pr & 86.47 / 87.87 & 93.46 / 92.56 & \textbf{93.59} / \textbf{92.65} & 89.99 / 90.27 \\
Ar$\rightarrow$Rw & 91.72 / 89.39 & 94.45 / 91.10 & \textbf{94.71} / \textbf{91.26} & 91.45 / 89.22 \\
Cl$\rightarrow$Ar & 90.55 / 82.28 & 93.41 / 83.84 & \textbf{93.48} / \textbf{83.87} & 93.17 / 83.71 \\
Cl$\rightarrow$Pr & 87.77 / 88.76 & 93.87 / 92.83 & \textbf{94.00} / \textbf{92.92} & 90.78 / 90.79 \\
Cl$\rightarrow$Rw & 92.67 / 90.00 & 95.16 / 91.54 & \textbf{95.50} / \textbf{91.76} & 91.95 / 89.54 \\
Pr$\rightarrow$Ar & 84.40 / 78.80 & 92.80 / 83.51 & 92.83 / 83.53 & \textbf{93.97} / \textbf{84.13} \\
Pr$\rightarrow$Cl & 82.86 / 79.24 & 89.63 / 83.24 & 88.76 / 82.74 & \textbf{90.24} / \textbf{83.59} \\
Pr$\rightarrow$Rw & 92.51 / 89.89 & 96.45 / 92.34 & \textbf{96.79} / \textbf{92.54} & 95.11 / 91.51 \\
Rw$\rightarrow$Ar & 86.97 / 80.28 & 93.14 / 83.69 & 93.12 / 83.68 & \textbf{94.05} / \textbf{84.18} \\
Rw$\rightarrow$Cl & 84.25 / 80.08 & 90.08 / 83.51 & \textbf{90.16} / \textbf{83.55} & 89.47 / 83.15 \\
Rw$\rightarrow$Pr & 91.76 / 91.44 & 95.68 / 94.00 & \textbf{95.88} / \textbf{94.13} & 92.90 / 92.20 \\
\bottomrule
\end{tabular}
\end{table}

\begin{table}[h]
\centering
\caption{Per-transfer OPDA results on DomainNet, averaged over three seeds for each transfer. Each entry reports H / H$^3$.}
\label{tab:opda-domainnet-per-transfer}
\small
\begin{tabular}{lcccc}
\toprule
Transfer & CLIP dis + entropy-only & partial-W & entropy-based & prototype-support \\
\midrule
painting$\rightarrow$real & 79.70 / 81.60 & \textbf{84.83} / \textbf{85.12} & 84.81 / 85.10 & 84.13 / 84.64 \\
painting$\rightarrow$sketch & 75.75 / 70.35 & 78.10 / 71.69 & \textbf{78.37} / \textbf{71.84} & 78.06 / 71.66 \\
real$\rightarrow$painting & 75.01 / 73.49 & 74.32 / 73.05 & 75.24 / 73.64 & \textbf{76.16} / \textbf{74.22} \\
real$\rightarrow$sketch & 76.86 / 70.99 & 77.55 / 71.38 & 78.23 / 71.76 & \textbf{78.59} / \textbf{71.97} \\
sketch$\rightarrow$painting & 75.44 / 73.76 & 75.28 / 73.67 & 76.22 / 74.26 & \textbf{76.59} / \textbf{74.50} \\
sketch$\rightarrow$real & 81.55 / 82.89 & 85.10 / 85.30 & \textbf{85.31} / \textbf{85.44} & 84.73 / 85.05 \\
\bottomrule
\end{tabular}
\end{table}
\clearpage
\section{Why the slack model is not a standard Sinkhorn problem}
\label{app:entropy-limitations}

This appendix makes precise the limitation mentioned in the discussion: while the linear slack formulation is exactly equivalent to a balanced OT problem on an augmented support, this equivalence does not survive standard entropic regularization in a form that would preserve the intended selective-rejection semantics.

Let
\[
\phi(t):= t(\log t - 1),\qquad t\ge 0,
\]
with the convention $\phi(0)=0$.
Consider first the natural entropically regularized slack objective
\begin{equation}
\label{eq:entropic-slack}
\begin{aligned}
\min_{P,u,v}\quad &
\langle C,P\rangle + \langle c_s,u\rangle + \langle c_t,v\rangle \\
&\qquad + \varepsilon \Bigl(\sum_{i,j}\phi(P_{ij}) + \sum_i \phi(u_i) + \sum_j \phi(v_j)\Bigr) \\
\text{s.t.}\quad &
P\mathbf 1 + u = \mu,\qquad P^\top \mathbf 1 + v = \nu,\qquad P,u,v\ge 0.
\end{aligned}
\end{equation}
This is the direct entropic analogue of Eq.~\eqref{eq:apot-slack}: the transport variables and the residual variables are regularized, but no new modeling term is added.

\begin{proposition}
\label{prop:entropic-not-equivalent}
Let $\bar\mu,\bar\nu,\bar C$ be the augmented marginals and cost matrix from Appendix~\ref{app:equivalence}, and let
\[
\bar P=
\begin{pmatrix}
P & u\\
v^\top & e
\end{pmatrix},
\qquad
e:=\sum_{i,j} P_{ij}.
\]
Then
\[
\langle \bar C,\bar P\rangle + \varepsilon \sum_{a,b}\phi(\bar P_{ab})
=
\Bigl[
\langle C,P\rangle + \langle c_s,u\rangle + \langle c_t,v\rangle
+ \varepsilon \Bigl(\sum_{i,j}\phi(P_{ij}) + \sum_i \phi(u_i) + \sum_j \phi(v_j)\Bigr)
\Bigr]
+ \varepsilon\,\phi(e).
\]
Consequently, the standard entropically regularized balanced OT problem on the augmented support is \emph{not} equivalent to Eq.~\eqref{eq:entropic-slack} unless $\varepsilon=0$.
\end{proposition}

\begin{proof}
The block decomposition of $\bar P$ gives
\[
\sum_{a,b}\phi(\bar P_{ab})
=
\sum_{i,j}\phi(P_{ij}) + \sum_i \phi(u_i) + \sum_j \phi(v_j) + \phi(e).
\]
Combining this identity with the linear objective preservation proved in Appendix~\ref{app:equivalence},
\[
\langle \bar C,\bar P\rangle
=
\langle C,P\rangle + \langle c_s,u\rangle + \langle c_t,v\rangle,
\]
yields the claimed formula immediately. The extra term $\varepsilon \phi(e)$ depends on the total transported mass $e=\sum_{i,j}P_{ij}$ and has no counterpart in Eq.~\eqref{eq:entropic-slack}. Therefore the two optimization problems coincide only in the unregularized case $\varepsilon=0$.
\end{proof}

Proposition~\ref{prop:entropic-not-equivalent} is the first obstruction to a direct Sinkhorn reduction: standard balanced entropic OT on the augmented support solves a \emph{different} objective, with an additional nonlinear bias on the transported mass.

The second issue is more structural. Even if one accepts the augmented entropic problem as a surrogate, the dummy row and dummy column carry macroscopic masses. At the empirical level, this creates a severe scale mismatch between dummy points and real points.

\begin{proposition}
\label{prop:entropic-dummy-scale}
In the augmented formulation from Appendix~\ref{app:equivalence}, the dummy source and dummy target carry masses
\[
\bar\mu_{n+1}=\|\nu\|_1,
\qquad
\bar\nu_{m+1}=\|\mu\|_1.
\]
In particular, if $\|\mu\|_1=\|\nu\|_1=1$ and the empirical measures are uniform on $n$ source points and $m$ target points, then
\[
\mu_i=\frac{1}{n},
\qquad
\nu_j=\frac{1}{m},
\qquad
\bar\mu_{n+1}=\bar\nu_{m+1}=1.
\]
Hence the dummy source is $m$ times heavier than each real target point, and the dummy target is $n$ times heavier than each real source point.
\end{proposition}

\begin{proof}
The formulas for $\bar\mu_{n+1}$ and $\bar\nu_{m+1}$ are exactly the definitions of the augmented marginals in Appendix~\ref{app:equivalence}. Under uniform empirical weights with unit total mass, each real point carries mass $1/n$ or $1/m$, while each dummy carries total mass $1$, which gives the claimed ratios immediately.
\end{proof}

This is the practical reason why the standard entropic dummy-point reduction behaves poorly. The issue is not only that entropic smoothing eliminates the exact admissibility rule of Proposition~\ref{prop:dominated}. It is that the two dummy channels act as large mass reservoirs, whereas each real sample carries only a tiny fraction of that mass.

For completeness, note that with finite costs the entropically regularized augmented problem still has a strictly positive kernel, so every real point interacts with the dummy row and the dummy column. Indeed, the unique minimizer $\bar P^{(\varepsilon)}$ satisfies
\[
\bar P^{(\varepsilon)}_{ab}
=
\exp\!\left(\frac{\alpha_a+\beta_b-\bar C_{ab}}{\varepsilon}\right)
>0
\qquad\text{for all } a,b
\]
for suitable dual scaling factors $\alpha_a,\beta_b$. Thus the large dummy marginals are not localized on a few rejected points: under entropic smoothing they spread diffusely across the whole support.

The next result shows that this is not just a small numerical leakage. In the high-entropy regime, the dummy channels absorb a macroscopic amount of mass independently of the geometry.

\begin{proposition}
\label{prop:entropic-high-eps}
Let $\bar P^{(\varepsilon)}$ be the minimizer of the entropically regularized augmented problem above, and let
\[
M := \|\bar\mu\|_1 = \|\bar\nu\|_1 = \|\mu\|_1 + \|\nu\|_1.
\]
Then, as $\varepsilon \to +\infty$,
\[
\bar P^{(\varepsilon)}
\longrightarrow
\bar P^{(\infty)}
:=
\frac{1}{M}\,\bar\mu\,\bar\nu^\top .
\]
In particular,
\[
\bar P^{(\infty)}_{i,m+1}
=
\frac{\mu_i\,\|\mu\|_1}{\|\mu\|_1+\|\nu\|_1},
\qquad
\bar P^{(\infty)}_{n+1,j}
=
\frac{\|\nu\|_1\,\nu_j}{\|\mu\|_1+\|\nu\|_1},
\]
and
\[
\bar P^{(\infty)}_{n+1,m+1}
=
\frac{\|\mu\|_1\,\|\nu\|_1}{\|\mu\|_1+\|\nu\|_1}.
\]
When $\|\mu\|_1=\|\nu\|_1=m$, this reduces to
\[
\bar P^{(\infty)}_{i,m+1}=\frac{\mu_i}{2},
\qquad
\bar P^{(\infty)}_{n+1,j}=\frac{\nu_j}{2},
\qquad
\bar P^{(\infty)}_{n+1,m+1}=\frac{m}{2}.
\]
\end{proposition}

\begin{proof}
Divide the augmented entropic objective by $\varepsilon$. One obtains
\[
\frac{1}{\varepsilon}\langle \bar C,\bar P\rangle + \sum_{a,b}\phi(\bar P_{ab}),
\]
minimized over the compact polytope $\Pi(\bar\mu,\bar\nu)$. As $\varepsilon\to+\infty$, the cost term vanishes uniformly, so every accumulation point of $\bar P^{(\varepsilon)}$ minimizes $\sum_{a,b}\phi(\bar P_{ab})$, or equivalently maximizes the Shannon entropy, over $\Pi(\bar\mu,\bar\nu)$.

The unique maximum-entropy coupling with fixed marginals is the independent coupling
\[
\bar P^{(\infty)}=\frac{1}{M}\bar\mu\,\bar\nu^\top,
\]
which can be verified directly either from strict concavity of entropy or from the identity
\[
\mathrm{KL}\!\left(\bar P \,\middle\|\, \frac{1}{M}\bar\mu\,\bar\nu^\top\right)
=
\sum_{a,b}\bar P_{ab}\log \bar P_{ab}
- \sum_a \bar\mu_a \log \bar\mu_a
- \sum_b \bar\nu_b \log \bar\nu_b
+ M \log M,
\]
whose last three terms are constant over $\Pi(\bar\mu,\bar\nu)$. Therefore $\bar P^{(\varepsilon)} \to \bar P^{(\infty)}$. The formulas for the dummy row and column are then obtained by reading off the corresponding entries of $\bar P^{(\infty)}$.
\end{proof}

Propositions~\ref{prop:entropic-not-equivalent}--\ref{prop:entropic-high-eps} clarify the sense in which IC-POT is not directly a Sinkhorn model. The issue is not merely algorithmic. Standard balanced entropic regularization on the augmented support both modifies the objective and introduces two dummy marginals whose masses are macroscopically larger than those of individual real points. Because entropic couplings remain diffuse, these dummy reservoirs interact with the whole support, and in the high-entropy regime they route a geometry-independent fraction of the mass through the dummy row and column. This is precisely the opposite of the intended selective-rejection semantics, where unmatched behavior should remain pointwise controlled and structurally sparse.

The same distinction separates IC-POT from unbalanced OT with spatially varying KL penalties \citep{liero2018optimal,chizat2018scaling}. A weighted KL-UOT model would take the form
\[
\min_{P\ge 0}\;
\langle C,P\rangle
+
\sum_i \tau_i\,\mathrm{KL}\!\bigl((P\mathbf 1)_i \,\big|\, \mu_i\bigr)
+
\sum_j \sigma_j\,\mathrm{KL}\!\bigl((P^\top\mathbf 1)_j \,\big|\, \nu_j\bigr),
\]
with pointwise penalty weights $\tau_i,\sigma_j \ge 0$. Writing $u_i=\mu_i-(P\mathbf 1)_i$ and $v_j=\nu_j-(P^\top\mathbf 1)_j$, the source-side penalty becomes
\[
\tau_i\,\mathrm{KL}(\mu_i-u_i \mid \mu_i)
=
\tau_i\Bigl[(\mu_i-u_i)\log\!\Bigl(1-\frac{u_i}{\mu_i}\Bigr)+u_i\Bigr],
\]
and similarly on the target side. This dependence is strictly convex in the rejected mass and explicitly anchored to the local reference masses $\mu_i,\nu_j$. By contrast, IC-POT uses the linear pricing terms $c_s(i)u_i + c_t(j)v_j$, together with the pointwise dual caps $f_i\le c_s(i)$ and $g_j\le c_t(j)$ from Proposition~\ref{prop:cs}. The two models therefore encode different objects: weighted KL-UOT implements a soft, mass-relative marginal relaxation, whereas IC-POT implements explicit pointwise rejection prices with hard acceptance thresholds. Spatially varying KL weights can make UOT more heterogeneous, but they do not recover the slack model studied here.
\end{document}